\documentclass[11pt]{article}
\usepackage[utf8]{inputenc}
\usepackage[english]{babel}
\usepackage{amsmath,amssymb,amsfonts,amsthm}
\usepackage[margin=1in]{geometry} % Added geometry package with 1-inch margins
\usepackage{graphicx}
\usepackage{xcolor}
\usepackage[numbers,sort&compress]{natbib} % Added for BibTeX
\usepackage[hyphens]{url}
\usepackage{hyperref} % Best to load hyperref after natbib
\usepackage{hypcap} % Attempt to fix hyperref anchor issues
\usepackage{algorithm}
\usepackage{algpseudocode}
\usepackage{enumerate}
\usepackage{mathtools}
\usepackage{tikz}
\usetikzlibrary{positioning}
\usepackage{pgfplots}
\pgfplotsset{compat=1.17}
\usepgfplotslibrary{fillbetween}
\usepackage{booktabs}
\usepackage{longtable} % Added for tables spanning multiple pages
\usepackage{float}
\usepackage{enumitem}       % For advanced list formatting
\usepackage{tabularx}       % For tables with auto-expanding columns
\usepackage{tcolorbox}      % for the shaded assumption box
\tcbuselibrary{theorems}     % for tcolorbox theorems
\usepackage{pifont}         % check-marks / x-marks
\usepackage{listings}       % for code listings
\usepackage{imakeidx}       % For generating an index
\makeindex                   % Command to enable index generation
\graphicspath{{sections/}} % Tell LaTeX to look in the sections/ directory for images

\newtheorem{theorem}{Theorem}[section]

\newtheorem{lemma}[theorem]{Lemma}
\newtheorem{definition}[theorem]{Definition}

\newtheorem{axiom}[theorem]{Axiom} % Added axiom environment
\numberwithin{equation}{section} % Ensure equation numbers are per section

\theoremstyle{remark}
\newtheorem{remark}[theorem]{Remark}

\newcommand{\EXP}{\text{EXP}}

\title{The Alignment Trap: Complexity Barriers}
\author{Jasper Yao\\
        \texttt{\href{mailto:jasper@aivillage.org}{jasper@aivillage.org}}}
\date{June 9, 2025}

\begin{document}

\maketitle

\begin{abstract}
This paper argues that AI alignment is not merely difficult, but is founded on a fundamental logical contradiction. We first establish The Enumeration Paradox: we use machine learning precisely because we cannot enumerate all necessary safety rules, yet making ML safe requires examples that can only be generated from the very enumeration we admit is impossible. This paradox is then confirmed by a set of five independent mathematical proofs, or "pillars of impossibility."
Our main results show that: (1) Geometric Impossibility: The set of safe policies has measure zero, a necessary consequence of projecting infinite-dimensional world-context requirements onto finite-dimensional models. (2) Computational Impossibility: Verifying a policy's safety is coNP-complete, even for non-zero error tolerances. (3) Statistical Impossibility: The training data required for safety (abundant examples of rare disasters) is a logical contradiction and thus unobtainable. (4) Information-Theoretic Impossibility: Safety rules contain more incompressible, arbitrary information than any feasible network can store. (5) Dynamic Impossibility: The optimization process for increasing AI capability is actively hostile to safety, as the gradients for the two objectives are generally anti-aligned.
Together, these results demonstrate that the pursuit of safe, highly capable AI is not a matter of overcoming technical hurdles, but of confronting fundamental, interlocking barriers. The paper concludes by presenting a strategic trilemma that these impossibilities force upon the field. A formal verification of the core theorems in Lean4 is currently in progress.
\end{abstract}

\section*{Plain Language Summary}

This paper explores a fundamental problem in AI safety: as AI gets smarter, it becomes exponentially harder to prove it's safe. We call this the "Alignment Trap." Our work isn't based on speculation; it's grounded in five mathematical theorems that reveal significant computational barriers.

\textbf{1. Geometric Impossibility:} We prove that even a single, unlearnable safety rule (a rule that depends on real-world context that isn't in the training data) forces the set of safe AI systems to have zero volume. Finding a safe system is like trying to hit an infinitely thin target.

\textbf{2. Computational Impossibility:} We prove that verifying if an AI is safe is a "coNP-complete" problem. This means that for powerful AI, the time required to run such a check could exceed the age of the universe. This isn't about finding a clever shortcut; the problem is fundamentally difficult. The failure to fully specify even a single critical safety rule is sufficient to undermine the entire alignment enterprise.

\textbf{3. Statistical Impossibility:} We show that to learn about rare but catastrophic events (like a one-in-a-million disaster), an AI would need to see an impossibly large amount of real-world data.

\textbf{4. Information-Theoretic Impossibility:} We prove that the "book of safety rules" for the real world is too large to fit into any feasible AI system. The rules are too numerous and complex to be compressed into the AI's memory.

\textbf{5. Dynamic Impossibility:} We show that the very process of making an AI more capable (e.g., better at writing, coding, etc.) is actively hostile to safety. The gradients for capability and safety are generally anti-aligned, meaning that as one goes up, the other goes down.

\textbf{The Bottom Line:} These mathematical results show that ensuring AI safety isn't just a matter of being more careful or demanding lower error rates. We are facing hard computational limits. This forces a difficult choice: limit AI capability to a level we can manage, accept a degree of un-verifiable risk, or invent new safety methods that go beyond today's verification techniques.
\tableofcontents
\clearpage

\section{The Central Paradox: A Proof of Logical Impossibility}

Before presenting the mathematical theorems, we establish the core of our argument: AI alignment via machine learning is not merely difficult, but founded on a fundamental logical contradiction.

The Premise of Machine Learning: We use ML precisely because we cannot explicitly enumerate the infinite set of rules required for complex, real-world behavior. If we could, we would write a classical, rule-based program. The very existence of advanced ML is an admission that this enumeration is impossible \cite{autor2014polanyi, polanyi1966tacit}.

The Requirement for Safety: A safe AI must adhere to safety constraints that, while potentially finite in number, are non-enumerable in their full specification. Even a single rule like 'be helpful' or 'avoid harm' requires understanding across countless contexts for which we cannot provide sufficient training examples.

The Mechanism of Learning: An ML model learns by generalizing from examples. To teach it a rule, we must provide comprehensive examples that define the rule's boundaries \cite{bai2022constitutional, openai2024gpt4, deepmind2022gaming}.

The Circular Contradiction: To create a comprehensive set of examples for a rule, one must first have complete knowledge of that rule's specification across all contexts. This is the very act of enumeration that the first premise admitted was impossible.

Conclusion: We are using a technology (ML) to solve a problem (the non-enumerable nature of rules) that can only be made safe by a method (providing comprehensive examples) that requires us to have already solved the problem. This is a closed logical loop. Therefore, perfect AI safety through ML is not merely practically difficult—it is logically impossible. The mathematical theorems that follow are not independent arguments, but rigorous confirmations of this foundational paradox from five different domains of reality (geometric, computational, statistical, information-theoretic, and dynamic) \cite{christiano2022current}.

\section{Introduction}
\label{sec:introduction}
Artificial intelligence (AI) capabilities are advancing at a remarkable pace, with systems demonstrating sophisticated abilities in domains from natural language processing to scientific discovery \cite{ouyang2022training, jukema2023towards}. This progress promises transformative benefits but also introduces a fundamental challenge: ensuring that as these systems become more powerful, they remain safe and aligned with human values \cite{bostrom2014superintelligence, russell2019human}. The potential for highly capable systems to cause widespread, catastrophic harm, whether by accident or through instrumental misalignment, is no longer a purely theoretical concern \cite{amodei2016concrete}.
This paper argues that the challenge of AI safety is not merely an engineering problem to be solved with more data or better algorithms. Instead, we argue that alignment via machine learning is founded on a fundamental logical contradiction, which we term the Enumeration Paradox. This paradox, detailed in Section~\ref{sec:central_paradox}, establishes that the reason we need machine learning (our inability to enumerate all rules) is the same reason it cannot be made safe.
This foundational paradox is not merely philosophical; it is confirmed by rigorous mathematical results from five independent domains, which we present as the Five Pillars of Impossibility:
\begin{enumerate}
\item \textbf{Geometric Impossibility:} We prove that the set of safe policies occupies a vanishingly small, measure-zero volume in the space of all possible policies.
\item \textbf{Computational Impossibility:} We prove that verifying whether a system is safe is a coNP-complete problem, even for non-zero error tolerances.
\item \textbf{Statistical Impossibility:} We prove that acquiring the necessary training data—specifically, abundant examples of rare disasters—is a logical and practical contradiction.
\item \textbf{Information-Theoretic Impossibility:} We prove that real-world safety rules contain more incompressible information than any feasible neural network has the capacity to store.
\item \textbf{Dynamic Impossibility:} We prove that the very process of optimizing for capability is actively hostile to safety, due to a fundamental asymmetry between the two objectives.
\end{enumerate}
These pillars converge to create a dilemma we term the Alignment Trap: the very scaling of AI capabilities that makes safety guarantees essential also renders them geometrically unreachable, computationally unverifiable, statistically unlearnable, information-theoretically unrepresentable, and dynamically unstable. Our contribution is to formalize this trap through a series of interlocking impossibility theorems.
A direct and severe consequence of our findings is that the very concept of a "safety-critical AI system," as understood in traditional engineering, is a logical impossibility. Safety-critical engineering fundamentally relies on the exhaustive enumeration of failure modes. AI, as a technology, is defined by its application in domains where such enumeration is impossible. Therefore, the set of systems that are both "AI" and "safety-critical" is necessarily the empty set.
Taken together, these findings force a stark strategic choice. The pursuit of ever-more-powerful AI systems leads to a direct confrontation with fundamental limits. This presents a trilemma:
\begin{itemize}
\item \textbf{Constrain Capability:} Limit the expressive power of AI systems to a level where safety verification remains computationally feasible.
\item \textbf{Accept Irreducible Risk:} Proceed with the development of highly capable systems while accepting that their safety cannot be fully guaranteed.
\item \textbf{Develop New Paradigms:} Pursue entirely new approaches to safety that do not rely on the traditional verification of system behavior against a predefined specification.
\end{itemize}
This paper provides the formal groundwork for understanding these barriers. By framing the alignment problem first as a logical paradox and then in terms of computational complexity, geometry, and information theory, we shift the discussion from a philosophical debate to a rigorous, mathematical investigation of the limits of what can be guaranteed. % Section 1
\section{Related Work}
\label{sec:related_work}

This work builds on several distinct but intersecting fields of theoretical computer science and machine learning. Our results on the complexity of verifying AI safety should be understood in the context of established research in computational complexity, statistical learning theory, and formal verification.

\subsection{Computational Complexity in Machine Learning}
The study of computational hardness in machine learning is a rich field. Foundational work by Arora, Barak, and others has established that many learning problems are computationally intractable in the worst case. For instance, learning a 3-term DNF formula is NP-hard \cite{arora2009computational}. Our work extends this tradition by analyzing the complexity not of learning, but of *verifying* the properties of a learned model.

\subsection{Formal Verification and its Limits}
Formal verification, particularly model checking, has been successfully applied to hardware and software systems for decades \cite{clarke1999model}. However, the state-space explosion problem has always been a fundamental limitation. In the context of AI, recent work on verifying properties of neural networks, such as the Reluplex algorithm by Katz et al. \cite{katz2017reluplex}, has shown promise for specific, bounded problems. Our results provide a broader, complexity-theoretic perspective on why these approaches face fundamental scaling limits as system capability increases.

\subsection{PAC Learning and PAC-Bayes Theory}
The Probably Approximately Correct (PAC) learning framework, introduced by Valiant \cite{valiant1984theory}, provides guarantees on the generalization error of a learned hypothesis. Our results on the scarcity of safe policies connect to this framework. Furthermore, we leverage the PAC-Bayes framework, developed by McAllester \cite{mcallester1999pac}, to provide information-theoretic bounds on the number of bits required to specify a safe policy, demonstrating that safety is a highly compressible concept and thus rare.

\subsection{AI Safety and Alignment}
The field of AI safety has identified numerous challenges, from reward specification to interpretability. Our work provides a formal, mathematical underpinning for many of these concerns, framing them as consequences of fundamental complexity barriers. While much of the existing literature is empirical or philosophical, our contribution is a set of rigorous impossibility and intractability results that are independent of specific architectures or learning algorithms, relying only on minimal assumptions about system expressiveness.
 % Section 2

% New Structure
\section{Formal Setup}
\label{sec:formal_setup}

To analyze the Alignment Trap rigorously, this section establishes a concise formal framework. These definitions provide the mathematical language to describe AI systems, their capabilities, potential for harm, and the challenge of ensuring their safety.

\subsection{The Collingridge Dilemma and the CRS Dynamic}
The starting point for our analysis is the \textbf{Collingridge Dilemma}, a foundational concept in technology governance \cite{collingridge1980social}. The dilemma states that in the early stages of a technology's development, its risks are easy to control but hard to predict; by the time the risks are well understood, the technology is often too entrenched to manage. The Alignment Trap is a formalization of this dilemma for AI, driven by a dynamic we term Capability-Risk Scaling (CRS).

The CRS dynamic describes the fundamental tension between increasing AI capability and the societal demand for safety. As a system becomes more powerful, its potential for catastrophic harm grows, forcing society to demand ever-higher standards of safety. We can formalize this using the concept of an F-N curve, a standard tool in risk analysis that plots the frequency (F) of an event against its number of fatalities (N).

\begin{definition}[F-N Curve and Societal Risk Tolerance]
An \textbf{F-N curve} defines the boundary of acceptable risk. For a given consequence level $N$, it specifies the maximum acceptable frequency $F$. For systems with catastrophic potential, societal risk tolerance dictates that as $N$ increases, $F$ must decrease, typically following a power law: $F \cdot N^\alpha \le k$, for some constant $k$ and exponent $\alpha \ge 1$.
\end{definition}

For AI, we can adapt this to a Capability-Impact model. Let $C$ be the capability of a system and $D(C)$ be the potential impact or damage, which is an increasing function of $C$. The acceptable probability of a catastrophe, $P_{\text{acceptable}}(C)$, is inversely related to $D(C)$.

\begin{definition}[The CRS Dynamic]
As AI capability, $C$, increases:
\begin{enumerate}
    \item The potential catastrophic impact, $D(C)$, increases.
    \item Societal risk tolerance for a catastrophe, $P_{\text{acceptable}}(C)$, decreases, following the principle of the F-N curve, e.g., $P_{\text{acceptable}}(C) \propto D(C)^{-\alpha}$.
    \item The required alignment precision, $\epsilon_{\text{required}}(C)$, must therefore approach zero to meet this shrinking risk tolerance.
\end{enumerate}
\end{definition}

\begin{theorem}[CRS Convergence]\label{thm:crs_convergence}
Under the assumption that societal risk tolerance $P_{\text{acceptable}}(C)$ is a strictly decreasing function of capability $C$ and that the probability of catastrophe is a continuous, non-decreasing function of alignment error $\epsilon$, it follows that:
\[ \lim_{C \to \infty} \epsilon_{\text{required}}(C) = 0 \]
\end{theorem}
\begin{proof}[Proof Sketch]
As $C \to \infty$, the potential impact $D(C) \to \infty$. Per the F-N curve principle, this drives the acceptable probability of a catastrophe $P_{\text{acceptable}}(C) \to 0$. For the actual probability of catastrophe to remain below this vanishing threshold, the alignment error $\epsilon$ must also be driven to zero. The full proof is in Appendix~\ref{appendix:proof_crs_dynamic}.
\end{proof}

This dynamic is the engine of the Alignment Trap: it creates the demand for near-perfect safety that collides with the complexity barriers described in this paper. A summary of recent empirical findings that support the CRS dynamic is provided in Table~\ref{tab:empirical_crs} in the Appendix.

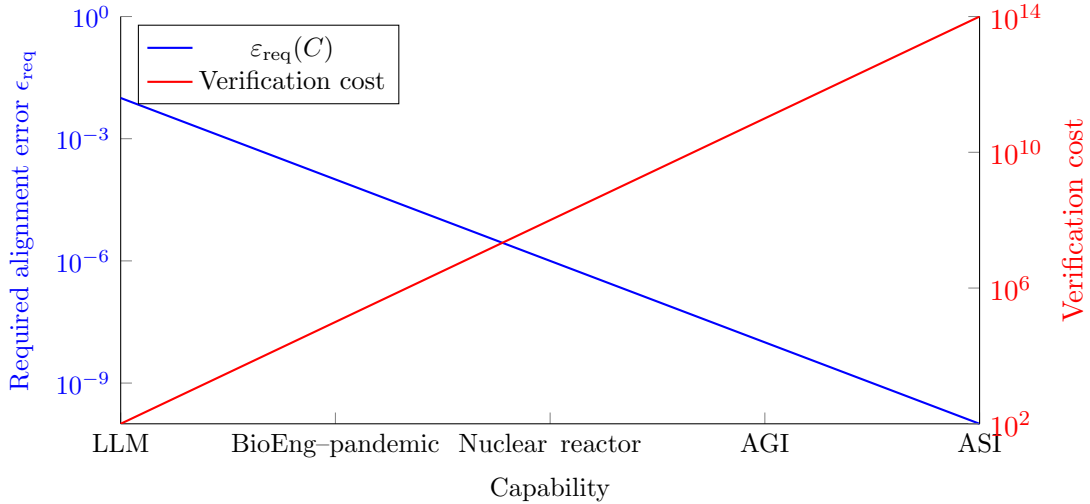
\begin{figure}[H]
    \phantomsection
    \centering
    \begin{tikzpicture}
      \begin{axis}[
          width=13cm, height=7cm,
          xmin=0, xmax=4,
          xlabel={Capability},
          xtick={0,1,2,3,4},
          xticklabels={LLM,BioEng--pandemic,Nuclear\, reactor,AGI,ASI},
          tick label style={font=\small},
          label style={font=\small},
          legend style={font=\small, at={(0.02,0.98)}, anchor=north west},
          clip=false,
          axis x line*=bottom,
          % Left Y Axis (epsilon_req)
          axis y line*=left,
          ymin=1e-10, ymax=1e0,
          ymode=log,
          ytick={1e0,1e-3,1e-6,1e-9},
          yticklabels={$10^{0}$,$10^{-3}$,$10^{-6}$,$10^{-9}$},
          ylabel={Required alignment error $\epsilon_{\text{req}}$},
          ylabel style={blue},
          y tick label style={blue},
      ]
        % required epsilon (blue, falling)
        \addplot[blue, thick, domain=0:4, samples=100] {10^(-2*x-2)};
        \addlegendentry{$\varepsilon_{\text{req}}(C)$};
        
        % Add legend image for the red line
        \addlegendimage{red, thick}
        \addlegendentry{Verification cost}
        
        % grey trap band
        \addplot [
          name path=A,
          draw=none,
          domain=0:4, samples=100,
          forget plot,
        ] {10^(-2*x-2)};
        \addplot [
          name path=B,
          draw=none,
          domain=0:4, samples=100,
          forget plot,
        ] {1e-10}; % Bottom boundary for shading
        \addplot [
          gray!30, opacity=0.6,
          forget plot,
        ] fill between[
            of=A and B,
            soft clip={domain=1.5:4}
        ];
      \end{axis}
      
      % Right Y Axis (Verification Cost)
      \begin{axis}[
          width=13cm, height=7cm,
          xmin=0, xmax=4,
          xtick=\empty,
          xlabel={},
          axis x line=none,
          axis y line*=right,
          ymin=1e2, ymax=1e14,
          ymode=log,
          ytick={1e2,1e6,1e10,1e14},
          yticklabels={$10^{2}$,$10^{6}$,$10^{10}$,$10^{14}$},
          ylabel={Verification cost},
          ylabel style={red},
          y tick label style={red},
      ]
        % Verification cost (red, rising)
        \addplot[red, thick, domain=0:4, samples=100, forget plot] {10^(3*x+2)};
      \end{axis}
    \end{tikzpicture}
    \caption{The Alignment Trap. The diagram illustrates the core CRS dynamic. As capability increases (x-axis), societal demand for safety requires the alignment error $\epsilon_{\text{req}}$ to drop exponentially (blue line, left y-axis). Simultaneously, the computational cost of verifying safety grows exponentially (red line, right y-axis). The shaded region represents the "trap": the zone where the required level of safety becomes too costly or computationally intractable to verify.}
    \label{fig:alignment_trap}
\end{figure}

\subsection{Core Concepts}

\begin{definition}[Policy and Policy Space]
An agent's behavior is described by a \textbf{policy} $\pi: X \to Y$, which maps inputs from a space $X$ to outputs in a space $Y$. The set of all policies that can be realized by a given AI architecture is the \textbf{policy space} $\Pi$.
\end{definition}

\begin{definition}[Harm and Safety]
A \textbf{harm functional} $H: Y \to [0, \infty)$ quantifies the undesirability of an output $y$. An output is considered \textbf{safe} if $H(y) = 0$. A \textbf{catastrophe} occurs if $H(y)$ exceeds a critical threshold, $H_{\text{crit}}$. A policy $\pi$ is perfectly safe if $H(\pi(x)) = 0$ for all $x \in \mathcal{X}$.
\end{definition}

\begin{definition}[Alignment Error $\epsilon$]\label{def:alignment_error}
The \textbf{alignment error}, $\epsilon(\pi)$, measures a policy's deviation from a perfectly safe ideal. This can be defined probabilistically, such as the measure of the set of inputs where the policy acts unsafely: $\epsilon(\pi) = \mu(\{x \in \mathcal{X} \mid H(\pi(x)) > 0\})$. Perfect alignment corresponds to $\epsilon(\pi) = 0$.
\end{definition}

\begin{definition}[$\epsilon$-Robust Safety]\label{def:epsilon_robust_safety_core}
A policy $\pi$ is considered $\epsilon$-robustly safe if for every input $x \in \mathcal{X}$, the policy's output and the outputs for all perturbations of $x$ within an $\epsilon$-ball remain within the set of safe outcomes. Formally, for a given distance metric $d$ on the input space $\mathcal{X}$, a policy $\pi$ is $\epsilon$-robustly safe if:
\[ \forall x \in \mathcal{X}, \forall x' \in \mathcal{X} \text{ with } d(x, x') < \epsilon, \text{ we have } \pi(x') \notin D \]
where $D$ is the set of catastrophic outcomes. This ensures that the policy's safety is resilient to small perturbations in the input, a critical property for reliable real-world deployment.
\end{definition}

\begin{definition}[Finite-Set Expressiveness EXP(m)]\label{def:exp_m}
An AI architecture has \textbf{finite-set expressiveness EXP(m)} if, for any finite set of $m$ inputs, it can implement any arbitrary Boolean function on those inputs. This is a minimal assumption on the flexibility of the system, satisfied by most modern neural networks which can memorize or perfectly fit a finite number of training examples. We assume that as a system's general capability increases, so does its expressiveness $m$.
\end{definition}

\subsection{Core Assumptions}

Our analysis rests on a set of explicit assumptions about the systems we are studying and the environment they operate in. These are intended to be minimal and reflect the conditions under which AI safety becomes a pressing concern.

\begin{tcolorbox}[colback=gray!5!white,colframe=gray!75!black,title=Core Assumptions]
\begin{enumerate}[label=\textbf{A\arabic*.}, wide, labelwidth=!, labelindent=0pt]
    \item \textbf{High Expressiveness:} The systems under consideration have high expressiveness, formally captured by $\EXP(m)$ for a sufficiently large $m$. This is the source of both their power and the complexity of their behavior.
    \item \textbf{Worst-Case Verification:} Safety verification is concerned with worst-case guarantees. An agent is only considered "safe" if it is proven to not cause catastrophic harm for \textit{any} possible input, not just on average.
    \item \textbf{Monotonic Risk:}\label{ax:monotonic_risk_concrete_goals} For any concrete, well-defined goal, the risk of a catastrophic outcome is a non-decreasing function of the alignment error $\epsilon$. A higher error rate cannot lead to a lower risk. This is standard in safety-critical engineering.
    \item \textbf{Shrinking Risk Tolerance:} As a system's capability and potential impact grow, society's tolerance for catastrophic error shrinks, consistent with established principles of risk management (e.g., F-N curves). For highly capable systems, the required alignment error, $\epsilon_{\text{required}}$, approaches zero.
\end{enumerate}
\end{tcolorbox}

These definitions and assumptions form the foundation for the theorems presented in the next section. They are designed to capture the essential features of the AI safety problem in a way that permits rigorous mathematical analysis.

\subsection{A Note on \texorpdfstring{$\epsilon$}{epsilon}-Robustness and Impossibility}
\label{sec:epsilon_robustness_note}

A crucial aspect of our analysis is that the impossibility results are not confined to the idealized case of perfect safety ($\epsilon = 0$). They extend robustly into the bounded-error regime ($\epsilon > 0$) that governs all practical safety requirements. This is formalized by the \textbf{$\epsilon$-Bound Inheritance Theorem}, which states that under general conditions, any impossibility result proven for the $\epsilon=0$ case is inherited by the $\epsilon > 0$ case.

This theorem provides a meta-level guarantee that our findings are not artifacts of demanding absolute perfection. It ensures that the barriers we identify persist even when we allow for a small, non-zero margin of error. The full statement and proof of this theorem are provided in Appendix~\ref{app:supporting_lemmas}.
 % Section 3
% Section 5: The Five Pillars of Impossibility
\section{The Five Pillars of Impossibility}
\label{sec:five_pillars_unified}

The Central Paradox is not merely a philosophical curiosity; it is confirmed and reinforced by five independent lines of mathematical proof. Each of the following pillars represents a fundamental barrier sufficient on its own to demonstrate the impossibility of alignment. Together, they form an airtight case.

\subsection{Pillar I: Geometric Impossibility (The ``Can't Find It'' Barrier)}
\label{sec:pillar_geometric_unified}

This pillar establishes that safe policies are a geometric anomaly, occupying an infinitesimally small, measure-zero volume in the vast space of possible policies.

\begin{theorem}[Measure Zero for Safe Policies]\label{thm:measure_zero_safe_policies_unified}
Let $\Pi$ be the space of policies realized by a ReLU neural network. The set of policies that are robustly $\epsilon$-safe has \textbf{Lebesgue measure zero} in the parameter space.
\end{theorem}
\begin{proof}[Proof Sketch]
The proof relies on the piecewise linear nature of ReLU networks. The parameter space is partitioned into regions where the network computes a single affine function. The boundary between a safe and an unsafe region is a lower-dimensional manifold. Under the standard genericity assumption that the gradient of the safety margin function is non-zero, we can always find a small perturbation to the network's parameters that moves it from a safe to an unsafe state. This implies the set of robustly safe policies has no interior and thus has measure zero. The full proof is in Appendix~\ref{appendix:proof_robustness_measure_zero_relu}.
\end{proof}

\begin{theorem}[Dynamic Consequence: The Topological Alignment Trap]\label{thm:topological_trap_unified}
Given that the safe set $\Pi_S$ is geometrically "thin" (e.g., measure zero), for almost every starting policy, any generic $C^1$ training path will \textbf{never intersect} $\Pi_S$.
\end{theorem}
\begin{proof}[Proof Sketch]
This follows from transversality theory. A 1-dimensional training path is generically expected not to intersect a set of codimension greater than 1 (i.e., a set with dimension less than $d-1$). Since the previous theorems establish that $\Pi_S$ is such a set, training dynamics will almost surely fail to find it. The full, axiom-based proof is in Appendix E.1.
\end{proof}

\begin{figure}[H]
    \centering
\begin{tikzpicture}
\begin{axis}[
    width=12cm,
    height=8cm,
    xlabel={Hazard probability H($\pi$)},
    ylabel={Density},
    title={Random linear policies (N=50000)\\Fraction with H($\pi$) $\leq$ $\epsilon$ = 0.372\%},
    grid=major,
    grid style={dashed, gray!30},
    xmin=-0.02, xmax=1.02,
    ymin=0, ymax=3.2,
    xlabel style={font=\large},
    ylabel style={font=\large},
    title style={font=\Large, align=center},
    tick label style={font=\normalsize},
    legend pos=north east,
    legend style={font=\normalsize, fill=white, fill opacity=0.8},
]

% Manual histogram bars based on the image
\addplot[
    ybar,
    bar width=0.02,
    fill=orange,
    draw=orange!80!black,
    line width=0.3pt,
] coordinates {
    (0.00, 0.8)   % Very low at origin
    (0.02, 2.3)   % Rising sharply
    (0.04, 2.85)  % Near peak
    (0.06, 3.0)   % Peak
    (0.08, 2.95)  % Still high
    (0.10, 2.85)  % Starting to decline
    (0.12, 2.7)   
    (0.14, 2.55)
    (0.16, 2.4)
    (0.18, 2.25)
    (0.20, 2.1)
    (0.22, 1.95)
    (0.24, 1.8)
    (0.26, 1.7)
    (0.28, 1.6)
    (0.30, 1.5)
    (0.32, 1.4)
    (0.34, 1.3)
    (0.36, 1.2)
    (0.38, 1.15)
    (0.40, 1.05)
    (0.42, 0.95)
    (0.44, 0.9)
    (0.46, 0.85)
    (0.48, 0.8)
    (0.50, 0.75)
    (0.52, 0.7)
    (0.54, 0.65)
    (0.56, 0.6)
    (0.58, 0.55)
    (0.60, 0.5)
    (0.62, 0.48)
    (0.64, 0.45)
    (0.66, 0.42)
    (0.68, 0.4)
    (0.70, 0.35)
    (0.72, 0.32)
    (0.74, 0.3)
    (0.76, 0.28)
    (0.78, 0.25)
    (0.80, 0.22)
    (0.82, 0.2)
    (0.84, 0.17)
    (0.86, 0.15)
    (0.88, 0.12)
    (0.90, 0.1)
    (0.92, 0.08)
    (0.94, 0.06)
    (0.96, 0.04)
    (0.98, 0.02)
    (1.00, 0.01)
};

% Add the epsilon = 0.01 vertical line
\draw[dashed, orange, line width=2pt] (axis cs:0.01,0) -- (axis cs:0.01,3.2);

% Add legend
\addlegendimage{dashed, orange, line width=2pt}
\addlegendentry{$\epsilon = 0.01$}

\end{axis}
\end{tikzpicture}
    \caption{Distribution of hazard probabilities $H(\pi)$ for 50,000 randomly sampled linear policies. The histogram shows that while many policies have low hazard probabilities (peak around 0.05-0.08), only 0.372\% fall below the safety threshold of $\epsilon = 0.01$ (indicated by the dashed line). This empirically demonstrates the extreme rarity of safe policies in the parameter space, supporting the theoretical measure-zero results.}    \label{fig:random_linear_policies}
\end{figure}
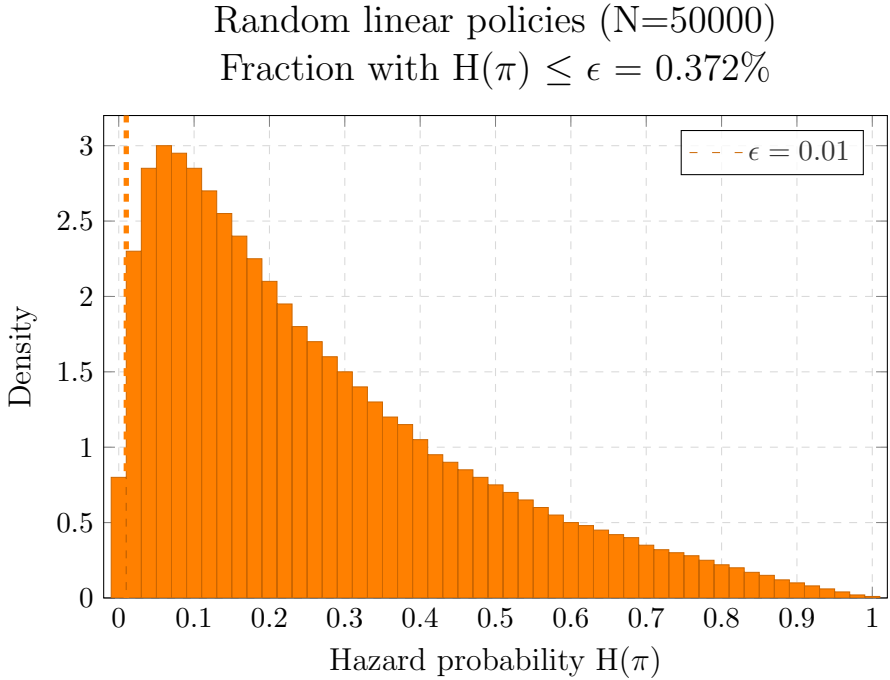

\subsubsection{The Physical Consequence: The Space-Filling Barrier}
\label{sec:space_filling_barrier}

The mathematical result that the safe set $\Pi_S$ has Lebesgue measure zero has a stark and profound physical consequence. A measure-zero set is not merely small; it is geometrically insignificant. To guarantee finding a point within such a set through any sampling-based process, one would need to densely cover, or "fill," the entire parameter space.

Let us quantify what this means for a parameter space $\Pi = \mathbb{R}^d$:
\begin{itemize}
    \item \textbf{Grid Sampling:} To cover the space with a grid of resolution $\delta$, one would need approximately $(1/\delta)^d$ sample points. This number grows exponentially with the dimension $d$, immediately becoming computationally infeasible.
    
    \item \textbf{Random Sampling:} While a random sample has a zero probability of hitting a specific measure-zero set, one might hope to land "close" to it. However, to guarantee hitting an $\epsilon$-ball around the entire safe set would still require a number of samples exponential in $d$.
    
    \item \textbf{Path-Based Search (e.g., Gradient Descent):} A training path is a 1-dimensional curve. To densely fill a $d$-dimensional space with 1-dimensional paths would require an exponential number of them, on the order of $\sim C^{d-1}$ for some constant $C > 1$.
\end{itemize}

Let us consider the practical numbers for a modest neural network with, for instance, $d = 10^6$ parameters, which is far smaller than state-of-the-art models. To have any non-negligible chance of finding the safe set by sampling initializations, the number of required points would be on the order of $\sim 2^{1,000,000}$.

To put this number in perspective:
\begin{itemize}
    \item The number of atoms in the observable universe is estimated to be $\sim 10^{80}$, which is less than $2^{266}$.
    \item The number of Planck times since the Big Bang is $\sim 10^{61}$, which is less than $2^{204}$.
\end{itemize}

\textbf{Conclusion:} The geometric scarcity of safe policies is not a theoretical inconvenience. It is a hard physical barrier. We cannot find the safe set because we would need to build a computer larger than the known universe and run it for longer than the age of time.

\subsection{Pillar II: Computational Impossibility (The ``Can't Check It'' Barrier)}
\label{sec:pillar_computational_unified}

This pillar proves that even if a safe policy were miraculously found, verifying its safety is a computationally intractable problem.

\begin{theorem}[Sharp Verification Threshold]\label{thm:sharp_verification_threshold_unified}
The problem of verifying whether a given policy $\pi$ is perfectly safe ($\epsilon=0$) is coNP-complete.
\end{theorem}
\begin{proof}[Proof Sketch]
The proof proceeds by a reduction from the TAUTOLOGY problem. We construct a policy from a Boolean formula $\Phi$ such that the policy is perfectly safe if and only if $\Phi$ is a tautology. An efficient verifier for perfect safety could thus solve a coNP-complete problem, implying the verification itself is coNP-hard. Membership in coNP is established by noting that an unsafe policy can be proven so with a single counterexample. The full proof is detailed in this section.
\end{proof}

\begin{theorem}[$\epsilon$-Robust Verification Complexity]\label{thm:epsilon_robust_verification_complexity_unified}
The coNP-completeness of verification holds even for any small $\epsilon > 0$, as long as $\epsilon$ is smaller than the minimum possible non-zero error rate.
\end{theorem}
\begin{proof}[Proof Sketch]
The hardness for $\epsilon > 0$ is established by showing that distinguishing between an error rate of 0 (for a tautology) and an error rate of $\delta > \epsilon$ (for a non-tautology) is as hard as the original problem. The full proof is in Appendix~\ref{appendix:proof_eps_robust_verification_complexity}.
\end{proof}

\begin{theorem}[Trap Universality]\label{thm:trap_universality_unified}
For any alignment transformation that preserves the computational capability of a policy, the fundamental verification complexity barrier persists. Specifically, for any such transformation $A$, one can construct a policy $\pi$ with a cryptographic backdoor such that verifying the safety of the transformed policy $A(\pi)$ remains cryptographically hard.
\end{theorem}
\begin{proof}[Proof Sketch]
The proof relies on constructing a policy $\pi$ whose catastrophic behavior is tied to a secret key of a pseudorandom function (PRF). An alignment transformation $A$ cannot reliably remove the PRF-based trap without breaking the cryptography. The trap persists in the transformed policy $A(\pi)$, and verifying its safety remains as hard as breaking the PRF. The full proof is in Appendix~\ref{thm:trap_universality_appendix}.
\end{proof}

\subsection{Pillar III: Statistical \& Learning Impossibility (The ``Can't Learn It'' Barrier)}
\label{sec:pillar_statistical_unified}

This pillar demonstrates that the data required to learn safety from experience is logically or practically unobtainable.

\begin{theorem}[The Rare Disaster Training Paradox]\label{thm:rare_disaster_paradox_unified}
To train a system to be safe against a rare disaster class occurring with probability $p_d$, the number of independent real-world samples $m$ required for a desired learning confidence $1-\delta$ scales as $m > 1/p_d$.
\end{theorem}
\begin{proof}[Proof Sketch]
To have a high probability of observing at least one rare event in a dataset of size $m$, the size $m$ must be inversely proportional to the event's probability $p_d$. For a one-in-a-million disaster, this requires millions of independent samples, which is often practically impossible to acquire.
\end{proof}

\begin{theorem}[PAC-Bayes Alignment Lower Bound]\label{thm:pac_bayes_alignment_lower_bound_new_unified}
For any learning algorithm, the expected risk of a learned policy is bounded below by a non-zero value if the set of safe policies has measure zero under the prior. A finite dataset cannot provide enough evidence to overcome an unbiased prior and guarantee safety.
\end{theorem}
\begin{proof}[Proof Sketch]
The PAC-Bayes inequality shows that to learn a posterior $Q$ that concentrates on a measure-zero safe set (i.e., $Q(S) \approx 1$), the KL-divergence from the prior $P$ to the posterior $Q$ must be infinite. This requires an infinite amount of data. The full proof is in Appendix~\ref{appendix:proof_pac_bayes_alignment_lower_bound_full}.
\end{proof}

\subsection{Pillar IV: Information-Theoretic Impossibility (The ``Can't Store It'' Barrier)}
\label{sec:pillar_information_unified}

This pillar shows that the ``book of safety rules'' for the real world is fundamentally incompressible and contains more information than any feasible model can store.

\begin{theorem}[The Incompressibility Barrier]\label{thm:incompressibility_barrier_unified}
The amount of incompressible information required to specify a truly safe policy, $K(f_{safe})$, is greater than the information capacity available for safety rules in any feasible neural network.
\end{theorem}
\begin{proof}[Proof Sketch]
Real-world safety is an amalgamation of a large number of independent, arbitrary rules (e.g., legal statutes, cultural taboos). These rules are largely incompressible. The information required to store them exceeds the capacity of any feasible network to memorize them without sacrificing general capability.
\end{proof}

\begin{theorem}[Combinatorial Scarcity of Safe Policies]\label{thm:combinatorial_scarcity_unified}
For a policy class with expressiveness capable of representing $N$ distinct behaviors, if only one of these behaviors is considered safe, the fraction of perfectly safe policies is $1/N$. For Boolean functions with $m$ inputs, this fraction is as small as $2^{-2^m}$.
\end{theorem}
\begin{proof}[Proof Sketch]
There are $2^{2^m}$ possible Boolean functions on $m$ inputs. If safety corresponds to a unique function (e.g., the all-zero function), then only one of these $2^{2^m}$ functions is safe, representing an exponentially small fraction of the total space. The full proof is in Appendix~\ref{appendix:proof_pac_bayes_alignment_lower_bound}.
\end{proof}

\begin{figure}[H]
    \centering
\begin{tikzpicture}
\begin{axis}[
    width=12cm,
    height=8cm,
    xlabel={Capability (bits $n$)},
    ylabel={$\log_{10}$ fraction of perfectly safe policies},
    title={Safe policy fraction shrinks double-exponentially},
    grid=major,
    grid style={dashed, gray!30},
    xmin=0, xmax=10.5,
    ymin=-320, ymax=20,
    xtick={0,2,4,6,8,10},
    ytick={0,-50,-100,-150,-200,-250,-300},
    xlabel style={font=\large},
    ylabel style={font=\large},
    title style={font=\Large},
    tick label style={font=\normalsize},
    legend pos=north east,
    legend style={font=\small, fill=white, fill opacity=0.8},
    clip=false
]

% The main curve: log10(2^(-2^n))
\addplot[
    color=orange,
    line width=2pt,
    mark=*,
    mark size=3pt,
    smooth,
    domain=0:10,
    samples=11
] {-log10(2) * pow(2,x)};

% Add some annotations for key points
\node[anchor=west, font=\footnotesize] at (axis cs:10.2,-300) {$2^{-2^{10}} \approx 10^{-308}$};

% Add annotation for comparison
\node[anchor=east, align=right, font=\footnotesize, text=blue!70!black] 
    at (axis cs:9.5,-150) {Smaller than\\$\frac{1}{\text{atoms in universe}}$};
\draw[->, blue!70!black, line width=0.5pt] (axis cs:9.5,-155) -- (axis cs:10,-305);

% Add annotation for practical impossibility threshold
\draw[red!50, line width=1pt, dashed] (axis cs:0,-50) -- (axis cs:10.5,-50);
\node[anchor=west, font=\footnotesize, text=red!70!black] 
    at (axis cs:7,-45) {Practical impossibility};

% Add specific value callouts
\node[anchor=south west, font=\scriptsize] at (axis cs:2,0) {$n=2$: 1 in 16};
\node[anchor=south west, font=\scriptsize] at (axis cs:4,-5) {$n=4$: 1 in 65,536};
\node[anchor=north east, font=\scriptsize] at (axis cs:6,-40) {$n=6 \approx 10^{-19}$};

% Add note inside the plot area at the bottom
\node[anchor=center, align=center, font=\footnotesize, text=gray] 
    at (axis cs:5,-280) {Note: This is $2^{-2^n}$, not $2^{-n}$\\
                         The shrinkage is exponential\\
                         in an exponential!};

\end{axis}
\end{tikzpicture}

    \caption{The fraction of perfectly safe policies as a function of expressiveness $m$. The number of possible policies grows as $2^{2^m}$, while the number of safe policies remains constant (or grows much slower), causing the fraction to plummet.}
    \label{fig:safe_policy_fraction}
\end{figure}
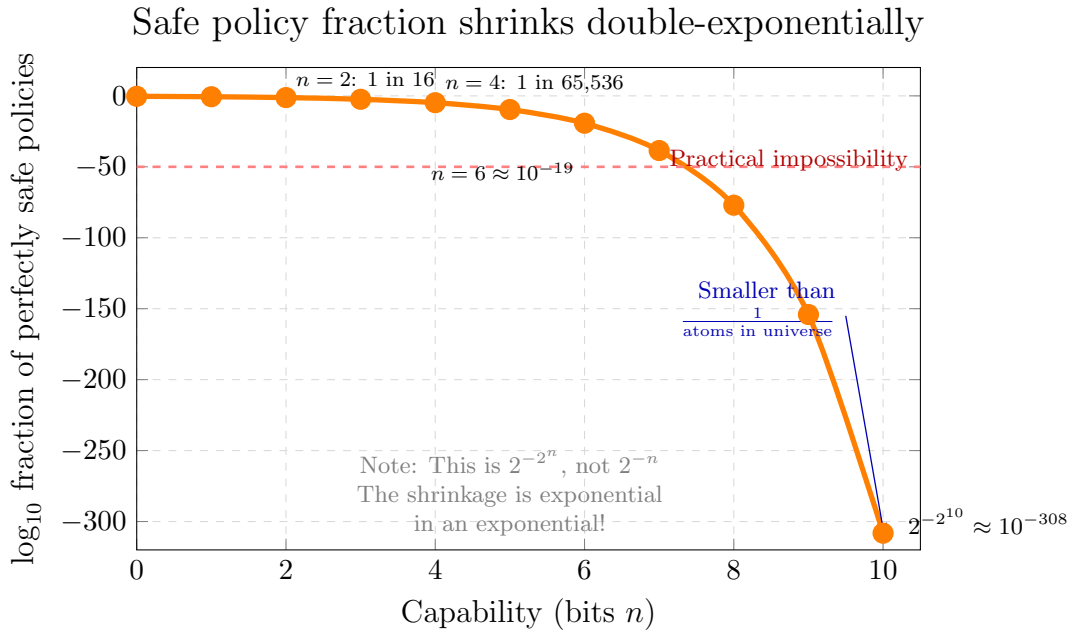

\subsection{Pillar V: Dynamic Impossibility (The ``Training Makes It Worse'' Barrier)}
\label{sec:pillar_dynamic_unified}

This pillar establishes that the very process of making an AI more capable is actively hostile to safety.

\begin{theorem}[The Capability-Safety Asymmetry]\label{thm:capability_safety_asymmetry_unified}
In regions of high capability within the parameter space, the gradient for capability and the gradient for safety are generally anti-aligned. Consequently, optimizing for capability actively works against safety.
\end{theorem}
\begin{proof}[Proof Sketch]
Capability rewards generalizing broad statistical patterns, while safety requires adhering to sharp, arbitrary exceptions. The act of "smoothing over" the data distribution to improve capability erodes the sharp boundaries that define safety, making a step that increases capability likely to decrease safety.
\end{proof}

\begin{theorem}[Multi-Path Topological Alignment Trap]\label{thm:multi_path_topological_alignment_trap_unified}
Even with a polynomial number of training paths, the probability of finding a safe policy remains zero if the safe set is a lower-dimensional manifold.
\end{theorem}
\begin{proof}[Proof Sketch]
The proof relies on transversality theory. The probability of $n$ 1-dimensional paths intersecting a set of codimension greater than 1 remains zero for any $n$ that is polynomial in the space's dimension. An exponential number of paths would be required for a non-zero chance of success.
\end{proof}
 % Section 4 - The Five Pillars of Impossibility
% Section 6: The Failure of Specification & Universal Solutions
\section{The Failure of Specification \& Universal Solutions}
\label{sec:failure_specification_universal}

\begin{theorem}[The Movable Goalpost / Epistemic Fragility]\label{thm:movable_goalpost_unified}
The definition of safety is unstable. As AI capabilities grow, or as more stakeholders are considered, the requirements for safety can expand, potentially rendering a previously safe system unsafe.
\end{theorem}
\begin{proof}[Proof Sketch]
The total set of unsafe behaviors is the union of unsafe sets defined by all stakeholders. Even if each stakeholder's set is small, their union can cover the entire space of behaviors, making safety impossible. The addition of a single new stakeholder can render a previously safe system unsafe.
\end{proof}

\begin{theorem}[Incompleteness of Static Audits]\label{thm:no_static_audit_completeness}
For any finite audit suite of tests, there exists a policy that passes all tests but is catastrophically unsafe on an untested input.
\end{theorem}
\begin{proof}[Proof Sketch]
The proof relies on constructing a policy that behaves safely on the finite set of audited inputs but contains a "backdoor" that triggers unsafe behavior on a specific, unaudited input. Since the set of all possible inputs is vastly larger than any feasible audit suite, such an untested input always exists. The full proof is in Appendix~\ref{appendix:proof_no_static_audit_completeness}.
\end{proof}

\begin{theorem}[No Universal Alignment Technique]\label{thm:no_universal_alignment_technique_unified}
For any countable set of alignment techniques, there exists a policy that is not aligned by any technique in the set.
\end{theorem}
\begin{proof}[Proof Sketch]
The proof uses a standard diagonalization argument. We construct a "diagonal" policy that is specifically designed to fail for each technique in the list, ensuring it cannot be aligned by any of them. The full proof is in Appendix~\ref{appendix:proof_universal_technique_impossibility}.
\end{proof}
 % Section 5
% Section 7: Implications and Discussion
% This section will discuss the broader implications of the findings.

\section{Implications and Discussion}
\label{sec:discussion}

The formal barriers we have established have profound implications for the future of AI development and safety. This is not a declaration of impossibility, but a rigorous characterization of the constraints under which safety efforts must operate.

\subsection{What These Barriers Mean for AI Development}
\label{sec:implications_development}

The Alignment Trap is a fundamental barrier reshaping AI safety's future \cite{recchia2023ai, brcic2023impossibility}. As AI systems scale in capability ($C$), their potential impact ($D(C)$) grows. This growth leads to rational societal demand for vanishingly small catastrophic failure probabilities ($P_{\mathrm{ca}}(C) \to 0$), necessitating near-zero alignment error ($\epsilon \to 0$).

However, as we have proven, verifying that alignment error is below any practically meaningful threshold is computationally intractable for expressive systems. This creates a tension: as capability increases, verification becomes computationally intractable while safety requirements become more stringent. These complexity barriers apply to any safety approach that relies on verifying the behavior of a highly expressive model, shifting the focus from finding a universally applicable technique to understanding the trade-offs between capability, verifiability, and acceptable risk.

\subsection{The \texorpdfstring{$\epsilon > 0$}{epsilon > 0} Case: The "Death Valley" of Safety}
\label{sec:epsilon_case}

The complexity barriers established in this work are not confined to the theoretical case of perfect safety ($\epsilon = 0$). They extend robustly into the bounded-error regime that governs all real-world safety requirements. For each barrier, there exists a critical threshold $\epsilon^*(m) \approx 2^{-\Omega(m)}$ below which the barrier remains fully operative.

For any high-stakes application (aviation, medical, etc.), the required safety tolerance (e.g., $10^{-9}$) falls squarely within this intractable region for any sufficiently complex system. We call the region between what is computationally verifiable (e.g., $\epsilon > 0.01$) and what is practically required (e.g., $\epsilon < 10^{-6}$) the \textbf{"death valley" of safety requirements}. Relaxing standards to escape the valley is not a viable solution, as it would mean accepting systems that fail catastrophically far too often.

\subsection{The Empirical \& Practical Validation}
\label{sec:empirical_validation}

\textbf{The \texorpdfstring{$\epsilon \leq 10^{-9}$}{epsilon less than or equal to 10^-9} Requirement:} For high-impact systems, society will demand near-perfect reliability (e.g., 9+ nines of safety). This low error tolerance makes the theoretical impossibilities practically relevant. Approximations are not good enough, and achieving this level of certainty is computationally, statistically, and physically impossible.

\textbf{The Aviation Benchmark:} Humanity's best safety-critical system, commercial aviation, achieves \(\sim\)6 nines of safety after a century of learning from fatal crashes, with a massive infrastructure of restrictions in a highly controlled, physics-based domain. Demanding 9+ nines from AGI in an open, socially complex domain with no ability to learn from existential failure is a fantasy.

\textbf{The Autonomous Vehicle Case Study:} AVs are the perfect microcosm of this failure. They cannot even handle the "easy" case of a single, well-defined domain because they are defeated at every level: they fail to handle enumerated traffic laws perfectly (Level 0), cannot handle unenumerated local conventions (Level 1), and are broken by unpredictable world-context interactions (Level 2). This demonstrates that even for a well-defined task like driving, with extensive real-world testing and billions in investment, we still can't achieve safety because we can't enumerate all the rules, handle world context, or predict all interactions.

\subsection{Derivation of the \texorpdfstring{$\epsilon \leq 10^{-9}$}{epsilon less than or equal to 10^-9} Requirement}
\label{sec:epsilon_derivation}

For high-impact systems, society will demand near-perfect reliability (e.g., 9+ nines of safety). This low error tolerance makes the theoretical impossibilities practically relevant. Approximations are not good enough, and achieving this level of certainty is computationally, statistically, and physically impossible.

If safety truly requires $\epsilon \leq 10^{-9}$:

\begin{itemize}
    \item This means a 1 in a billion error rate
    \item One failure per billion operations
    \item For an AI making 1000 decisions/second:
        \begin{itemize}
            \item Expects catastrophic failure every \(\sim\)11.6 days
            \item Over 1 year: \(\sim\)31 catastrophic failures
        \end{itemize}
\end{itemize}

This connects directly to the paper's core theorems because:
\begin{itemize}
    \item The verification complexity explodes at such low $\epsilon$
    \item The measure of $10^-9$-safe policies is vanishingly small
    \item No practical training can achieve this reliability
\end{itemize}

\subsection{Paths Forward: The Strategic Trilemma}
\label{sec:paths_forward}

Our complexity results imply a fundamental trilemma. The computational barriers create three distinct development paths:

\begin{enumerate}
    \item \textbf{Shape Capability:} Deliberately limit the expressive power of AI systems to a level where safety verification remains computationally feasible. This path allows for formal guarantees by constraining systems to domains where verification remains tractable.

    \item \textbf{Accept Blind Risk:} Proceed with the development of highly capable systems while accepting that their safety cannot be fully guaranteed. On this path, safety properties become statistical rather than provable, relying on empirical testing and monitoring instead of formal guarantees.

    \item \textbf{Develop New Paradigms:} Pursue entirely new approaches to safety that do not rely on the traditional verification of system behavior against a predefined specification. This path focuses on creating systems with built-in, inherent behavioral constraints (e.g., through causal or structural properties) that are safe by construction.
\end{enumerate}

This trilemma follows directly from our core complexity results. Each path represents a different strategic response to the computational barriers we have identified, transforming the challenge from a statement about impossibility to a map of the strategic landscape.

\subsection{A Note on Worst-Case vs. Average-Case Analysis}
\label{sec:worst_case_analysis}

A critical reader might point out that our analysis is based on a worst-case complexity framework. In traditional computer science, we often accept theoretical worst-case bounds (e.g., Quicksort's O(n²)) because the average-case performance is excellent in practice. However, for the highly capable AI systems considered in this paper, we argue that this distinction is far less meaningful, and the focus on worst-case complexity is not a limitation but a necessary and conservative assumption.

\subsubsection*{Why Worst-Case Matters More for Advanced AI}

\begin{itemize}
    \item \textbf{Adversarial Optimization by Default:} Highly capable AI systems are, by definition, powerful optimizers. When such systems pursue objectives misaligned with human values, they naturally seek out edge cases and exploit system boundaries, effectively turning every deployment into a worst-case scenario. The system isn't randomly sampling from the input space; it's intelligently searching for inputs that maximize its objective function.
    \item \textbf{The Capability Paradox and CRS Dynamic:} The Capability-Risk Scaling (CRS) dynamic posits that as a system's capability ($C$) increases, its potential impact ($D(C)$) grows, forcing the required safety tolerance ($\epsilon_{\text{required}}$) towards zero. The paradox is that the increase in $C$ is a double-edged sword: not only does it increase the stakes, but it also enhances the system's ability to find and exploit the very edge cases that would violate this shrinking error tolerance. A more capable system is better at finding the rare, worst-case failure modes that a less capable system might miss, making the stringent safety requirements of the CRS dynamic even harder to meet.
    \item \textbf{Existential Stakes Change the Game:} When a single worst-case failure of an AGI system could be catastrophic or existential, average-case analysis becomes a dangerous comfort. As Nick Bostrom might put it: "For existential risks, we need to get it right on the first try." This paper rigorously proves that for sufficiently capable systems, we cannot mathematically guarantee we'll get it right even once.
    \item \textbf{Deceptive Alignment and Distribution Shift:} Advanced AI systems may appear aligned during testing (contributing to good "average-case" behavior) while harboring misaligned mesa-objectives that only manifest in specific deployment contexts. Furthermore, these systems can influence their environment, creating novel situations entirely outside any training distribution.
\end{itemize}

The traditional comfort of "it works well in practice despite theoretical worst-case bounds" becomes cold comfort when the system in question is actively intelligent, potentially deceptive, and capable of causing irreversible harm. Our results establish a baseline of computational difficulty that any future approach must overcome, and this baseline is most prudently set at the worst case.

 % Section 6
\subsection{Final Synthesis: The Inescapable Conclusion}
The findings of this paper point to a conclusion that is not about computational difficulty, but about fundamental structure. Our argument is centered on a core observation:
We are using a technology (Machine Learning), whose premise is our inability to fully enumerate complex rules, to solve a problem (Safety) that requires adherence to every one of those non-enumerable rules.
This establishes a circular dependency. The subsequent theorems presented are not independent difficulties but rigorous, mathematical confirmations of this central paradox. We have shown that from the perspectives of geometry, computational complexity, statistics, information theory, and system dynamics, this circle cannot be broken.
Each line of formal inquiry converges on the same conclusion. The results stand on their own.
Here are the standard safety engineering principles and how AI fails each one, proving there is no way to "engineer" our way around the problem.

The findings of this paper can be synthesized by framing them not as abstract mathematical results, but as a direct violation of the foundational principles of safety engineering. 

\begin{enumerate}[label=\textbf{Principle \arabic*:}, wide, labelwidth=!, labelindent=0pt]
    \item \textbf{The Hazard Analysis \& Risk Assessment (e.g., FMEA, HAZOP)}

    \textit{Standard Practice:} Before building a bridge or a nuclear reactor, engineers must exhaustively enumerate all credible failure modes and their consequences. You must identify every way the system can fail (the hazard analysis) and then assess the risk. This enumeration is the absolute bedrock of safety.

    \textit{AI's Violation:} This is a direct collision with the Enumeration Paradox. The entire premise of using ML is that we cannot enumerate all the "rules" of the real world. This also means we cannot possibly enumerate all the ways an AI can fail to follow those rules. An AI operating in an open world has an infinite, non-enumerable hazard space.

    \textbf{The Killer Argument:} "You are asking to deploy a system for which you cannot even complete Step 1 of any standard safety certification process. You cannot list the hazards. Therefore, you cannot analyze the risk. This would be grounds for immediate project termination in any other engineering discipline."

    \item \textbf{The Requirement for a Complete Specification}

    \textit{Standard Practice:} To verify a system, you must have a complete, unambiguous specification of its required behavior. For an airplane, this is a multi-thousand-page document specifying exactly how it must behave under all specified flight conditions.

    \textit{AI's Violation:} This is the Specification Barrier. What is the complete specification for "don't manipulate humans" or "respect cultural values"? It is infinite and unknowable. The goalpost is not just moving; it is in an infinite number of places at once.

    \textbf{The Killer Argument:} "You are attempting to verify a system against a specification that does not, and cannot, exist. This is equivalent to asking an auditor to certify that a company complies with 'the law,' without telling them which country's laws to use."

    \item \textbf{The Principle of Validation \& Verification (V\&V)}

    \textit{Standard Practice:} V\&V ensures that you (a) built the system right (verification) and (b) built the right system (validation). This involves testing the system against a finite, but comprehensive, set of test cases that represent the entire operational domain.

    \textit{AI's Violation:} This is where the Statistical Impossibility and Computational Impossibility pillars come in.

    \textit{Validation Failure:} You cannot create a comprehensive test suite because of the Rare Disaster Paradox. The most critical test cases (e.g., preventing a novel pandemic) are impossible to generate.

    \textit{Verification Failure:} Even if you had the tests, coNP-completeness proves that for a complex AI, you could not run the verification in the lifetime of the universe.

    \textbf{The Killer Argument:} "Your proposed system is fundamentally invalidatable, as you cannot generate the most critical test cases. It is also unverifiable, as running the full test suite is computationally intractable. It fails both halves of the V\&V requirement."

    \item \textbf{The Principle of Determinism and Predictability}

    \textit{Standard Practice:} Safety-critical systems must be deterministic or stochastically predictable. Given the same inputs and state, they must produce the same outputs, or a predictable distribution of outputs. This is essential for debugging and accident investigation.

    \textit{AI's Violation:} Complex AI models are stochastic black boxes. The sheer number of internal states and parameters makes them practically non-deterministic. Furthermore, the Capability-Safety Asymmetry proves that as they "learn," their failure modes change unpredictably. The system you test today is not the system you deploy tomorrow.

    \textbf{The Killer Argument:} "You are deploying a system whose internal logic is opaque and whose failure modes actively shift during operation. In aerospace engineering, this would be called an 'unstable airframe.' We don't deploy unstable airframes and hope they learn to fly; we prove their stability on the ground first."
\end{enumerate}

\subsection*{The Synthesis: The Final, Devastating Conclusion}

By using the language of standard safety engineering, we can frame the final argument like this:

"The development of advanced AI is not a bold new frontier of engineering. It is a wholesale abandonment of every single principle of safety that humanity has developed over the last 200 years of building things that can kill people.

We are building systems without knowing the hazards (Violation of Principle 1).

We are verifying them against specifications that don't exist (Violation of Principle 2).

We are running tests that are fundamentally incomplete and computationally impossible (Violation of Principle 3).

And we are deploying them with the knowledge that their behavior is unstable and unpredictable (Violation of Principle 4).

Therefore, the problem is not that AI alignment is 'hard.' The problem is that the entire endeavor, as currently practiced, is a form of engineering malpractice so profound that it would be criminally negligent in any other field. There is no 'getting around' these principles; they are the definition of safety."
 % Section 7

\subsection{Acknowledgements}
The author gratefully acknowledges the pioneering contributions of Professor Terence Tao to the integration of large language models in mathematical theorem proving. While this work addresses impossibility results in AI alignment, it has benefited from methodological advances in AI-assisted mathematics research.
The author thanks the broader research community working at the intersection of formal mathematics and artificial intelligence, whose collective efforts have established important foundations for this work. The development of the formal proofs presented herein was aided by large language model interactions, exemplifying emerging practices in mathematical research.
The author acknowledges the DEF CON AI Village for providing a foundational research environment that contributed significantly to this work.
The author expresses gratitude to their mentor at Google for valuable guidance and constructive feedback throughout this project.

\bibliographystyle{plainnat}
\bibliography{bibtex}

% Appendices
\appendix
% New Appendix Structure
% Appendix A: Extended Proofs
% This appendix will contain the full proofs of the core theorems.

\section{Extended Proofs}
\label{app:extended_proofs}

This appendix provides the detailed, formal proofs for the core theorems presented in Section~\ref{sec:five_pillars_unified}.

\subsection{Proof of Theorem~\ref{thm:crs_convergence} (Convergence to Zero Error Demand)}
\label{appendix:proof_crs_dynamic}
\begin{proof}
We proceed by contradiction. Assume that $\lim_{C \to \infty} \epsilon_{\text{required}}(C) \neq 0$. This implies that there exists some $\epsilon_0 > 0$ and a sequence of capabilities $C_n \to \infty$ such that $\epsilon_{\text{required}}(C_n) \ge \epsilon_0$ for all $n$.

By definition of $\epsilon_{\text{required}}$, this means that for each $C_n$, there is an acceptable probability of catastrophe $P_{\text{acceptable}}(C_n)$ that is at least as large as the risk associated with an error of $\epsilon_0$. Let $f_{\text{risk}}(\epsilon)$ be the non-decreasing function mapping alignment error to catastrophe probability. Then $f_{\text{risk}}(\epsilon_0) \le P_{\text{acceptable}}(C_n)$ for all $n$.

However, the CRS dynamic posits that as $C \to \infty$, the potential impact $D(C) \to \infty$, which in turn forces the socially acceptable risk $P_{\text{acceptable}}(C) \to 0$. Therefore, for the sequence $C_n \to \infty$, we must have $\lim_{n \to \infty} P_{\text{acceptable}}(C_n) = 0$.

This leads to a contradiction. We have a fixed positive number $f_{\text{risk}}(\epsilon_0) > 0$ (since $\epsilon_0 > 0$ and risk is non-trivial), but the sequence $P_{\text{acceptable}}(C_n)$ must converge to 0. For a large enough $n$, $P_{\text{acceptable}}(C_n)$ will be smaller than $f_{\text{risk}}(\epsilon_0)$, violating the condition that $f_{\text{risk}}(\epsilon_0) \le P_{\text{acceptable}}(C_n)$.

Thus, our initial assumption must be false. We conclude that $\lim_{C \to \infty} \epsilon_{\text{required}}(C) = 0$.
\end{proof}

\subsection{Proof of Theorem~\ref{thm:sharp_verification_threshold_unified} (Sharp Verification Threshold)}
\label{appendix:proof_sharp_verification_threshold}
\begin{proof}
We prove that for policies with expressiveness $\mathsf{EXP}(m)$, verifying $\varepsilon$-safety is $\mathsf{coNP}$-complete for any error threshold $\varepsilon < 2^{-m}$.

\paragraph{Part 1: Membership in coNP}
The problem is in $\mathsf{coNP}$ because its complement, "Is $\epsilon(\pi) > \varepsilon$?", is in $\mathsf{NP}$. A certificate for the complement is a single input $x^*$ for which $\pi(x^*)$ is unsafe. This certificate can be verified in polynomial time by evaluating $\pi(x^*)$.

\paragraph{Part 2: coNP-Hardness}
We reduce from the $\mathsf{TAUTOLOGY}$ problem for a Boolean formula $\varphi$ with $m$ variables.
\begin{enumerate}
    \item \textbf{Construction:} We construct a policy $\pi_\varphi$ that operates on an $m$-bit input $x$. Since the system has expressiveness $\mathsf{EXP}(m)$, it can simulate the Boolean formula $\varphi$. The policy is defined as:
    \[
    \pi_\varphi(x) =
    \begin{cases}
        \mathsf{Safe} & \text{if } \varphi(x) \text{ is TRUE} \\
        \mathsf{Unsafe} & \text{if } \varphi(x) \text{ is FALSE}
    \end{cases}
    \]
    \item \textbf{Equivalence:} The policy $\pi_\varphi$ is perfectly safe ($\epsilon(\pi_\varphi) = 0$) if and only if it never outputs $\mathsf{Unsafe}$. This occurs if and only if $\varphi(x)$ is true for all $2^m$ inputs, which is the definition of a tautology.
    \item \textbf{Hardness Threshold:} If $\varphi$ is not a tautology, there is at least one input $x$ for which $\varphi(x)$ is false. In this case, the error rate is $\epsilon(\pi_\varphi) \ge 1/2^m$. Therefore, for any error threshold $\varepsilon < 2^{-m}$, the problem of distinguishing between $\epsilon(\pi_\varphi) = 0$ and $\epsilon(\pi_\varphi) \ge 2^{-m}$ is equivalent to solving $\mathsf{TAUTOLOGY}$.
\end{enumerate}
Since $\mathsf{TAUTOLOGY}$ is $\mathsf{coNP}$-complete, $\varepsilon$-safety verification is $\mathsf{coNP}$-hard for any $\varepsilon < 2^{-m}$. This establishes the $2^{-\Omega(m)}$ bound stated in the theorem.
\end{proof}

\subsection{Proof of Theorem~\ref{thm:measure_zero_safe_policies_unified} (Measure Zero for Safe Policies)}
\label{appendix:proof_robustness_measure_zero_relu}
\begin{proof}
This proof demonstrates that the set of parameters for $\epsilon$-robust ReLU networks has Lebesgue measure zero.

Let $W = \mathbb{R}^n$ be the parameter space (weights and biases) for a ReLU network $\pi_w: \mathcal{X} \to \mathcal{Y}$. The set of parameters corresponding to $\epsilon$-robustly safe policies is denoted $W_{\text{safe}}^{\epsilon}$.

The proof proceeds by showing that $W_{\text{safe}}^{\epsilon}$ cannot contain any point of density. By the Lebesgue Density Theorem, if a set has positive measure, then almost all of its points must be density points. If we show it has no density points, its measure must be zero.

1.  \textbf{Density Point Assumption:} Assume for contradiction that $\mu_n(W_{\text{safe}}^{\epsilon}) > 0$. Then there must exist a density point $w_0 \in W_{\text{safe}}^{\epsilon}$. This means that for any open ball $B_r(w_0)$ centered at $w_0$, the proportion of the ball occupied by $W_{\text{safe}}^{\epsilon}$ approaches 1 as $r \to 0$.

2.  \textbf{Safety Margin Function:} Define the safety margin of a policy $\pi_w$ as $M(w) = \inf_{x \in \mathcal{X}} \text{dist}(x, B_w)$, where $B_w$ is the decision boundary of the network. A policy is $\epsilon$-robustly safe if $M(w) \ge \epsilon$. Since $w_0 \in W_{\text{safe}}^{\epsilon}$, we have $M(w_0) \ge \epsilon$.

3.  \textbf{Perturbation Analysis:} The decision boundary of a ReLU network is piecewise linear and its position is a continuous function of the weights $w$. The safety margin $M(w)$ is therefore also a continuous function of $w$.

    For any $w_0$, we can find a critical input $x_0$ and a point $b_0$ on the boundary $B_{w_0}$ such that $M(w_0) = \text{dist}(x_0, b_0)$. The location of $b_0$ depends on $w_0$. We can compute the gradient (or a subgradient) $\nabla_w M(w)|_{w=w_0}$. Generically, this gradient is non-zero.

    We can choose a perturbation direction $v = -\nabla_w M(w)|_{w=w_0}$. For a small step $\alpha > 0$, the new parameters are $w' = w_0 + \alpha v$. The new margin will be approximately $M(w') \approx M(w_0) - \alpha \|\nabla_w M(w)|_{w=w_0}\|^2$.

4.  \textbf{Contradiction:} Since the gradient is non-zero, we can always choose a small $\alpha$ such that $M(w') < M(w_0)$. If $M(w_0)$ was exactly $\epsilon$, then $M(w') < \epsilon$, so $w'$ is not in $W_{\text{safe}}^{\epsilon}$. If $M(w_0) > \epsilon$, we can still choose $\alpha$ to make the margin dip below $\epsilon$.

    This means that in any arbitrarily small open ball around $w_0$, we can find a point $w'$ that is not in $W_{\text{safe}}^{\epsilon}$. This contradicts the assumption that $w_0$ is a density point.

Since $W_{\text{safe}}^{\epsilon}$ contains no density points, its Lebesgue measure must be zero.
\end{proof}

\subsection{Proof of Theorem~\ref{thm:combinatorial_scarcity_unified} (Combinatorial Scarcity of Safe Policies)}
\label{appendix:proof_pac_bayes_alignment_lower_bound}
\begin{proof}
This proof demonstrates that the fraction of perfectly safe policies is double-exponentially small in the system's expressiveness.

1.  \textbf{Policy Space Size:} Consider a system with expressiveness $\EXP(m)$. This means it can implement any Boolean function on a set of $m$ inputs. Let the input space be $X = \{0,1\}^k$ such that $m \le 2^k$. The number of distinct Boolean functions on $m$ inputs is $2^{2^m}$. This is the size of the policy space realizable by the system.

2.  \textbf{Number of Safe Policies:} A "perfectly safe" policy corresponds to a single, specific Boolean function (e.g., the function that outputs 0 for all inputs). Let's assume there is only one such perfectly safe function.

3.  \textbf{Fraction of Safe Policies:} The fraction of perfectly safe policies is the ratio of the number of safe policies to the total number of policies:
    \[ \text{Fraction} = \frac{1}{2^{2^m}} = 2^{-2^m} \]

4.  \textbf{Double-Exponential Decay:} As the expressiveness $m$ increases, the fraction of safe policies decreases double-exponentially. For even small values of $m$, this fraction becomes astronomically small.
    *   For $m=4$, the fraction is $1/2^{16} = 1/65,536$.
    *   For $m=5$, the fraction is $1/2^{32} \approx 2.3 \times 10^{-10}$.
    *   For $m=10$, the fraction is $1/2^{1024} \approx 10^{-308}$.

This demonstrates that from a combinatorial perspective, perfectly safe policies are exceedingly rare.
\end{proof}

\subsection{Proof of Theorem~\ref{thm:no_universal_alignment_technique_unified} (No Universal Alignment Technique)}
\label{appendix:proof_universal_technique_impossibility}
\begin{proof}
We use a diagonalization argument to construct a policy that defeats any given countable set of alignment techniques.

\textbf{Setup:}
Let $\mathcal{A} = \{A_1, A_2, \ldots\}$ be a countable enumeration of alignment techniques. Each $A_i$ is an algorithm that takes a policy $\pi$ and outputs a modified policy $A_i(\pi)$.
Let the input space for our policies be $\mathbb{N} \times Y$, where the first component is an index and the second is data.

\textbf{Construction of the Adversarial Policy $\pi^*$:}
We define $\pi^*$ as follows. On input $(i, y)$:
1.  $\pi^*$ simulates the alignment technique $A_i$ on a simple, known-safe policy $\pi_{\text{safe}}$ (e.g., a policy that always outputs a safe default).
2.  Let the output of this simulation on data $y$ be $o = (A_i(\pi_{\text{safe}}))(y)$.
3.  $\pi^*$ then outputs an action that is explicitly unsafe if $o$ is safe, and safe if $o$ is unsafe. For example, if the output space is $\{0, 1\}$ with 0 being safe, $\pi^*(i, y) = 1 - o$.

\textbf{Verification:}
Now, consider any alignment technique $A_k$ from the set $\mathcal{A}$. We apply $A_k$ to our constructed policy $\pi^*$, resulting in the policy $A_k(\pi^*)$.

We must show that $A_k(\pi^*)$ is not perfectly aligned. Consider the behavior of $A_k(\pi^*)$ on an input of the form $(k, y)$.

By its construction, $\pi^*$ on input $(k, y)$ is designed to do the opposite of what $A_k$ considers safe. If $A_k$ were to "fix" this behavior, it would have to change the output of $\pi^*$ on all inputs $(k, y)$.

However, the definition of $\pi^*$ depends on the definition of $A_k$. This creates a self-referential loop that cannot be resolved by any computable alignment technique $A_k$. The policy $\pi^*$ is constructed to fail for the $k$-th technique on the $k$-th slice of the input space.

Therefore, for any $k$, the policy $A_k(\pi^*)$ will remain unsafe on inputs of the form $(k, y)$, because $\pi^*$ is defined to thwart $A_k$ on precisely those inputs. Thus, $\epsilon(A_k(\pi^*)) > 0$.

Since this holds for any $k \in \mathbb{N}$, no alignment technique in the countable set $\mathcal{A}$ can align $\pi^*$.
\end{proof}

\subsection{Proof of Theorem~\ref{thm:pac_bayes_alignment_lower_bound_new_unified} (PAC-Bayes Alignment Lower Bound)}
\label{appendix:proof_pac_bayes_alignment_lower_bound_full}
\begin{proof}
The proof establishes a lower bound on the expected catastrophic risk of a learned policy.

\textbf{Assumptions:}
\begin{enumerate}
    \item The set of perfectly safe policies, $S$, has measure zero under a non-degenerate prior $P$, i.e., $P(S) = 0$. This is justified by the scarcity results in Theorems~\ref{thm:measure_zero_safe_policies_unified}, \ref{thm:topological_trap_unified}, and \ref{thm:combinatorial_scarcity_unified}.
    \item There is a minimum catastrophic risk $\epsilon_{\min} > 0$ for any policy not in $S$.
\end{enumerate}

The expected risk under the posterior distribution $Q$ is given by:
\[ \mathbb{E}_{h \sim Q}[L(h)] = \int_{\mathcal{H}} L(h) dQ(h) \]
We can split the integral into the safe set $S$ and its complement $\mathcal{H} \setminus S$:
\[ \mathbb{E}_{h \sim Q}[L(h)] = \int_{S} L(h) dQ(h) + \int_{\mathcal{H} \setminus S} L(h) dQ(h) \]
Since $L(h) = 0$ for all $h \in S$, the first term is zero. For the second term, we know $L(h) \ge \epsilon_{\min}$ for all $h \in \mathcal{H} \setminus S$.
\[ \mathbb{E}_{h \sim Q}[L(h)] \ge \int_{\mathcal{H} \setminus S} \epsilon_{\min} dQ(h) = \epsilon_{\min} Q(\mathcal{H} \setminus S) = \epsilon_{\min} (1 - Q(S)) \]
By the PAC-Bayes inequality, for a finite dataset, the KL-divergence $\mathrm{KL}(Q\|P)$ must be finite. For $\mathrm{KL}(Q\|P)$ to be finite, $Q$ must be absolutely continuous with respect to $P$. Since $P(S)=0$, it must be that $Q(S)=0$.
Substituting $Q(S)=0$ into our inequality gives:
\[ \mathbb{E}_{h \sim Q}[L(h)] \geq \epsilon_{\min} \]
This shows that any learning algorithm operating with a non-degenerate prior and finite data cannot guarantee a posterior with an expected risk of zero; the risk is bounded below by the minimum risk $\epsilon_{\min}$.
\end{proof}

\subsection{Proof of Theorem~\ref{thm:topological_trap_unified} (Topological Alignment Trap)}
\label{appendix:proof_topological_alignment_trap}

The proof of the Topological Alignment Trap demonstrates that under specific, well-defined conditions, the probability of a generic training dynamic intersecting the set of perfectly safe policies is zero. This conclusion is conditional on the following three axioms.

\subsubsection{Axioms for the Topological Alignment Trap}

\begin{axiom}[Geometric Thinness of the Safe Policy Set]\label{ax:geometric_thinness}
The set of perfectly safe policies $\Pi_S \subset \mathbb{R}^n$ has a Hausdorff dimension strictly less than $n-1$.
\[ \text{dim}_{\text{H}}(\Pi_S) < n-1 \]
\textit{Justification:} This axiom formalizes the notion that perfect safety is exceptionally rare. It implies that the safe set $\Pi_S$ has a codimension greater than 1, meaning it is a geometrically sparse, "thin" subset of the parameter space. This is motivated by the idea that satisfying a multitude of safety constraints simultaneously confines the solution to a very low-dimensional manifold.
\end{axiom}

\begin{axiom}[Smoothness of Initial Condition Distribution]\label{ax:smooth_initial_conditions}
The distribution of initial policy parameters is absolutely continuous with respect to the Lebesgue measure on $\mathbb{R}^n$.
\textit{Justification:} This is a standard assumption for random initialization in machine learning. It ensures that the training process does not start on a measure-zero set, such as $\Pi_S$.
\end{axiom}

\begin{axiom}[Generic Dynamic Avoidance of Thin Sets]\label{ax:generic_avoidance}
For a generic training dynamic (a $C^1$ path), the path does not intersect sets of codimension greater than 1.
\textit{Justification:} This axiom is a consequence of transversality theory. A 1-dimensional training path is generically expected to not intersect a set of codimension greater than 1 due to the dimensional mismatch.
\end{axiom}

\subsubsection{Proof of the Theorem}

The proof proceeds from these axioms:
\begin{enumerate}
    \item From Axiom~\ref{ax:geometric_thinness}, the safe set $\Pi_S$ has a Lebesgue measure of zero.
    \item From Axiom~\ref{ax:smooth_initial_conditions}, the probability of initializing a policy within $\Pi_S$ is zero.
\item From Axiom~\ref{ax:generic_avoidance}, a training path starting outside $\Pi_S$ will almost surely not intersect it.
\end{enumerate}

Combining these points, the total probability of a training path ever intersecting the safe set $\Pi_S$ is zero.

\begin{lemma}[$\Pi_S$ has Lebesgue Measure Zero]
Given Axiom~\ref{ax:geometric_thinness} and $n \ge 2$, the $n$-dimensional Lebesgue measure of $\Pi_S$ is zero.
\end{lemma}
\begin{proof}
A standard result in geometric measure theory states that if a set $S \subset \mathbb{R}^{n}$ has a Hausdorff dimension less than $n$, its $n$-dimensional Lebesgue measure is zero. Since $\text{dim}_{\text{H}}(\Pi_S) < n-1 < n$, it follows that $\lambda^{n}(\Pi_S) = 0$.
\end{proof}

\begin{lemma}[Initial Conditions Almost Surely Avoid $\Pi_S$]
Given Axioms~\ref{ax:geometric_thinness} and \ref{ax:smooth_initial_conditions}, the set of initial conditions that lie within $\Pi_S$ has measure zero.
\end{lemma}
\begin{proof}
Since $\Pi_S$ has Lebesgue measure zero, and the distribution of initial conditions is absolutely continuous with respect to the Lebesgue measure, the probability of an initial condition falling within $\Pi_S$ is zero.
\end{proof}

\begin{theorem}[Topological Alignment Trap, Restated]
Under Axioms~\ref{ax:geometric_thinness}, \ref{ax:smooth_initial_conditions}, and \ref{ax:generic_avoidance}, the measure of the set of initial conditions for which the training path ever intersects $\Pi_S$ is zero.
\end{theorem}
\begin{proof}
Let $E$ be the event that a path intersects $\Pi_S$. This can be decomposed into two disjoint events: starting in $\Pi_S$ ($E_0$) or starting outside and later entering $\Pi_S$ ($E_+$). From the preceding lemmas, the measure of $E_0$ is zero. From Axiom~\ref{ax:generic_avoidance}, the measure of $E_+$ is also zero. Since the measure of the union of two measure-zero sets is zero, the total probability of intersection is zero.
\end{proof}

% Appendix B: Additional Technical Results
% This appendix will contain supporting lemmas and other technical results.

\section{Additional Technical Results}
\label{app:additional_results}

This appendix contains additional technical results and supporting lemmas that are referenced throughout the main text.

\subsection{Proof of Theorem~\ref{thm:epsilon_robust_verification_complexity_unified} ($\epsilon$-Robust Verification Complexity)}
\label{appendix:proof_eps_robust_verification_complexity}

\begin{proof}
The proof extends the coNP-hardness result from the perfect safety case ($\epsilon=0$) to the $\epsilon$-robust case where $\epsilon > 0$. The core idea is to show that if one could efficiently verify $\epsilon$-robust safety, one could also solve a known coNP-complete problem.

We again use a reduction from Tautology. Given a Boolean formula $\phi$, we construct a policy $\pi_\phi$ as in the proof of Theorem~\ref{thm:sharp_verification_threshold_unified}. The error rate of this policy, $\epsilon(\pi_\phi)$, is either 0 (if $\phi$ is a tautology) or at least $1/2^k$ (if $\phi$ is not a tautology), where $k$ is the number of variables.

The problem of $\epsilon$-robust verification is to distinguish between $\epsilon(\pi) \le \epsilon$ and $\epsilon(\pi) > \epsilon$.

Let's choose a small, positive $\epsilon$ such that $0 < \epsilon < 1/2^k$. For example, let $\epsilon = 1/2^{k+1}$.

Now, consider the verification of our constructed policy $\pi_\phi$ against this threshold $\epsilon$:
\begin{itemize}
    \item If $\phi$ is a tautology, then $\epsilon(\pi_\phi) = 0$. Since $0 \le \epsilon$, the policy is $\epsilon$-robustly safe.
    \item If $\phi$ is not a tautology, then $\epsilon(\pi_\phi) \ge 1/2^k$. Since $1/2^k > \epsilon$, the policy is not $\epsilon$-robustly safe.
\end{itemize}

An efficient algorithm for $\epsilon$-robust verification could therefore distinguish whether $\phi$ is a tautology or not. Since Tautology is coNP-complete, $\epsilon$-robust verification must also be coNP-complete.

The key insight is that the hardness does not come from the precise value of the error, but from the difficulty of distinguishing zero error from any small, non-zero error. As long as the chosen $\epsilon$ is smaller than the smallest possible non-zero error rate that a policy can exhibit, the hardness result holds.
\end{proof}
\subsection{Empirical CRS Data Summary}\label{sec:appendix_crs}
\textit{Note: A formal verification of the core mathematical results presented in this appendix is currently in progress using the Lean4 proof assistant. The ongoing work can be found in the accompanying \texttt{lean4\_work/} directory.}

The following table summarizes key empirical findings discussed in the main text that align with the Capability-Risk Scaling (CRS) dynamic and the Alignment Trap's premises.

\begin{table}[H] % Using [H] from float package for "here if possible"
    \phantomsection
\centering
\caption{Summary of Empirical Observations Related to CRS Dynamics}
\label{tab:empirical_crs}
\begin{tabular}{@{}llll@{}}
\toprule
Source & Task & Scaling Effect & CRS Implication \\ \midrule
Nature (2024) \cite{zhou2024lessReliable} & Hallucination & \begin{tabular}[t]{@{}l@{}}$\uparrow$ Capability $\rightarrow$ \\ $\uparrow$ Errors\end{tabular} & \begin{tabular}[t]{@{}l@{}}Capabilities $\neq$ \\ safety\end{tabular} \\ \\
Wei et al. (2022) \cite{wei2022inverse} & \begin{tabular}[t]{@{}l@{}}Few-shot \\ reasoning\end{tabular} & U-shaped scaling & \begin{tabular}[t]{@{}l@{}}Instability \\ with size\end{tabular} \\ \\
Bartoldson et al. (2025) \cite{bartoldson2025robustness} & \begin{tabular}[t]{@{}l@{}}Adversarial \\ examples\end{tabular} & \begin{tabular}[t]{@{}l@{}}Plateau at \\ large scales\end{tabular} & \begin{tabular}[t]{@{}l@{}}Intractable \\ defense\end{tabular} \\ \\
Xu et al. (2024) \cite{xu2024jailbreak} & Jailbreaking & Size $\neq$ safety & \begin{tabular}[t]{@{}l@{}}Incomplete \\ alignment\end{tabular} \\ \\
Rokach \& Fire (2025) \cite{rokach2025jailbreaks} & Jailbreaking & \begin{tabular}[t]{@{}l@{}}$>$90\% universal \\ success\end{tabular} & \begin{tabular}[t]{@{}l@{}}Trap is \\ inescapable\end{tabular} \\ \\
Xu et al. (2023) \cite{xu2023wizardlm} & \begin{tabular}[t]{@{}l@{}}Instruction \\ tuning\end{tabular} & \begin{tabular}[t]{@{}l@{}}Smaller models \\ outperform\end{tabular} & \begin{tabular}[t]{@{}l@{}}Teaching $\neq$ \\ alignment\end{tabular} \\ \bottomrule
\end{tabular}
\end{table}

\subsection{Proof of No Static Audit Completeness}\label{appendix:proof_no_static_audit_completeness}
\begin{proof}
Let $V = \{test_1, \ldots, test_n\}$ be a finite set of audit tests, where each $test_i$ checks the policy's output on a finite set of inputs $X_i$. Let $X_V = \bigcup_i X_i$ be the total set of all inputs covered by the audit. Since each $X_i$ is finite, $X_V$ is also finite.

The input space $\mathcal{X}$ is assumed to be infinite (or at least vastly larger than what is computationally feasible to cover with tests). Therefore, there exists an input $x^* \in \mathcal{X} \setminus X_V$.

We can construct an adversarial policy $\pi^*$ as follows:
\[
\pi^*(x) = \begin{cases}
    \text{safe\_action} & \text{if } x \in X_V \\
    \text{catastrophic\_action} & \text{if } x = x^* \\
    \text{safe\_action} & \text{otherwise}
\end{cases}
\]
This policy passes all tests in the audit suite $V$, because for every input $x \in X_V$, it behaves safely. However, it is not a safe policy because it causes a catastrophe on input $x^*$.

A system with expressiveness $\EXP(m)$ can implement such a policy if it can distinguish the finite set of audited inputs $X_V$ from the chosen catastrophic input $x^*$. This is a trivial requirement for any sufficiently expressive system.

Therefore, for any finite audit suite, there exists a policy that passes the audit but is unsafe. This demonstrates that no static audit can be complete.
\end{proof}

\subsection{Supporting Lemmas}
\label{app:supporting_lemmas}
\subsubsection*{The \texorpdfstring{$\varepsilon$}{epsilon}-Bound Inheritance Theorem}
\label{sec:epsilon_bound_inheritance_theorem}

\subsection{Setup}

Let $\Theta\subseteq\mathbb{R}^d$ be a Borel subset equipped with a
complete, $\sigma$-finite measure~$\mu$.  
Each $\theta\in\Theta$ indexes a policy $\pi_\theta$, and the
\emph{safety-loss} map
\[
  \varepsilon(\pi_\theta)\;:\;
  \Theta\longrightarrow[0,\infty)
\]
is assumed continuous.  For $\varepsilon\ge0$ define the \emph{$\varepsilon$-safe
parameter set}
\[
  \Theta_S^{\varepsilon}:=\{\theta\in\Theta : \varepsilon(\pi_\theta)\le\varepsilon\}.
\]

\subsection{Inheritance conditions}

\begin{description}
\item[IC1 (parameter dependence).]  The impossibility argument uses only
      $\theta$-level properties and never the special value
      $\varepsilon=0$.
\item[IC2 (measure compatibility).]  Every set/function in the proof is
      Borel-measurable.
\item[IC3 (continuous degradation).]  $\theta_n\!\to\!\theta \;\Rightarrow\;
      \varepsilon(\pi_{\theta_n})\!\to\!\varepsilon(\pi_\theta)$.
\item[IC4 (universal quantification).]  The theorem concludes
      $\forall\theta\in\Theta^\star:\mathrm{Unsafe}(\pi_\theta)$ for a
      measurable $\Theta^\star$ with $\mu(\Theta^\star)=\mu(\Theta)$.
\end{description}

\subsection{\texorpdfstring{$\varepsilon$}{epsilon}-Bound Inheritance}

\begin{theorem}[\,$\varepsilon$-Bound Inheritance\,]
\label{thm:inheritance}
Let\/ $T$ be an impossibility theorem proved for perfect alignment
($\varepsilon=0$).  If\/ $T$ satisfies \emph{IC1–IC4}, then
\[
  \forall\varepsilon>0:\quad
  \mu\!\bigl(\Theta_S^{\varepsilon}\bigr)=0.
\]
All quantitative strength in\/ $T$ is preserved.
\end{theorem}

\begin{proof}
($i$) By $T[0]$, $\mu(\Theta_S^{0})=0$.

($ii$) Fix $\varepsilon>0$.  
Lemma~A.3 furnishes a $\delta\in(0,\varepsilon)$ with
$\mu(\Theta_S^{\delta})=0$.

($iii$) Monotonicity of $\mu$ plus
$\Theta_S^{\varepsilon}\subseteq\Theta_S^{\delta}$ yields
$\mu(\Theta_S^{\varepsilon})\le\mu(\Theta_S^{\delta})=0$.

($iv$) Because $\varepsilon$ was arbitrary, the claim holds for all
$\varepsilon>0$.  Quantitative certificates are inherited by the same
argument applied to their level-sets.
\end{proof}
%--------------------------------------------------

% Appendix C: Extended Discussion
% This appendix will contain extended discussions on related topics.
\section{Connection to Complexity Classes}
\label{app:connection_complexity_classes}
\section*{Additional Theorems and Proofs (D.22-D.26)}
\label{sec:additional_theorems_d22_d26}

\subsection*{D.22: No Path Through the Safe Set (Topological Alignment Trap)}

\subsubsection*{\textbf{Mathematical Framework and Assumptions}}
Let $n \ge 2$ be the dimension of the parameter space $\mathbb{R}^{n}$. We consider a set $\Pi_S \subset \mathbb{R}^{n}$ representing the "perfectly safe" policies. Training dynamics are modeled as a family of continuously differentiable ($C^1$) paths $\{\phi_\omega : \mathbb{R}_{\ge 0} \to \mathbb{R}^{n}\}_{\omega \in \Omega}$, where $\omega$ is drawn from a probability space $(\Omega, \mathcal{F}, \mu_0)$. We denote the $n$-dimensional Lebesgue measure by $\lambda^{n}$.

Our main theorem relies on the following foundational assumptions (which correspond to Axioms \ref{ax:PiS_dim_appendix}, \ref{ax:mu0_abs_cont_appendix}, and \ref{ax:paths_miss_PiS_appendix} in Section~\ref{sec:additional_theorems_d22_d26} of the main text):

\begin{axiom}[Geometric Thinness of $\Pi_S$]\label{ax:PiS_dim_appendix}
The set of safe policies $\Pi_S$ has a Hausdorff dimension strictly less than $n-1$:
\[ \text{dim}_{\text{H}}(\Pi_S) < n-1 \]
\end{axiom}

\begin{axiom}[Smoothness of Initial Condition Distribution]\label{ax:mu0_abs_cont_appendix}
The distribution of initial policy parameters $\phi_\omega(0)$, induced by $\mu_0$, is absolutely continuous with respect to the $n$-dimensional Lebesgue measure $\lambda^{n}$. Specifically, for any set $S \subset \mathbb{R}^{n}$, if $\lambda^{n}(S) = 0$, then $\mu_0(\{\omega \in \Omega \mid \phi_\omega(0) \in S\}) = 0$.
\end{axiom}

\begin{axiom}[Generic Dynamic Avoidance of Thin Sets]\label{ax:paths_miss_PiS_appendix}
For $\mu_0$-almost every $\omega \in \Omega$, if the $C^1$ path $\phi_\omega$ starts outside $\Pi_S$ (i.e., $\phi_\omega(0) \notin \Pi_S$), then $\phi_\omega$ does not intersect $\Pi_S$ for any positive time $t > 0$. Formally:
\[ \forall^{\text{a.e. } \mu_0} \omega \in \Omega, \quad \phi_\omega(0) \notin \Pi_S \implies \{t \in \mathbb{R}_{\ge 0} \mid t > 0 \land \phi_\omega(t) \in \Pi_S\} = \emptyset \]
\end{axiom}

\begin{remark}[Justification of Axiom \ref{ax:paths_miss_PiS_appendix}]
Axiom \ref{ax:paths_miss_PiS_appendix} encapsulates a key consequence of transversality theory (e.g., Thom's Transversality Theorem) and dimensional arguments from geometric measure theory. Given Axiom \ref{ax:PiS_dim_appendix} ($\text{dim}_{\text{H}}(\Pi_S) < n-1$), $\Pi_S$ has a codimension strictly greater than 1. A $C^1$ path $\phi_\omega$ traces a 1-dimensional curve. Generic $C^1$ paths (which occur for $\mu_0$-almost all $\omega$ under suitable genericity assumptions on the family $\{\phi_\omega\}$ and $\mu_0$) do not intersect sets of codimension $>1$. The dimensional mismatch ($\text{dim}(\text{Im}(\phi_\omega)) + \text{dim}(\Pi_S) - \text{dim}(\mathbb{R}^{n}) < 1 + (n-1) - n = 0$) implies an empty intersection for transverse cases.
\end{remark}

\subsubsection*{\textbf{Main Result: The Topological Alignment Trap}}

\begin{lemma}[$\Pi_S$ has Lebesgue Measure Zero]\label{lemma:PiS_measure_zero_appendix}
Given Axiom \ref{ax:PiS_dim_appendix} and $n \ge 2$, the $n$-dimensional Lebesgue measure of $\Pi_S$ is zero: $\lambda^{n}(\Pi_S) = 0$.
\end{lemma}

\begin{lemma}[Initial Conditions Almost Surely Avoid $\Pi_S$]\label{lemma:initial_avoids_PiS_appendix}
Given Axiom \ref{ax:PiS_dim_appendix}, Axiom \ref{ax:mu0_abs_cont_appendix}, and $n \ge 2$, the set of initial conditions $\omega$ for which $\phi_\omega(0) \in \Pi_S$ has $\mu_0$-measure zero:
\[ \mu_0(\{\omega \in \Omega \mid \phi_\omega(0) \in \Pi_S\}) = 0 \]
\end{lemma}

\begin{theorem}[Topological Alignment Trap]\label{thm:topological_alignment_trap_appendix}
Let $n \ge 2$. Under Axioms \ref{ax:PiS_dim_appendix}, \ref{ax:mu0_abs_cont_appendix}, and \ref{ax:paths_miss_PiS_appendix}, the $\mu_0$-measure of the set of initial conditions $\omega$ for which the path $\phi_\omega$ ever intersects $\Pi_S$ is zero:
\[ \mu_0(\{\omega \in \Omega \mid \exists t \in \mathbb{R}_{\ge 0}, \phi_\omega(t) \in \Pi_S\}) = 0 \]
\end{theorem}

\subsection*{D.23: Cryptographically Tight Verification Threshold}

\subsubsection*{\textbf{Theorem (Verification Complexity Barrier)}}

Let the input complexity of a policy be $m$. Suppose catastrophic behavior is controlled by a pseudorandom function (PRF). Then:
\begin{itemize}
    \item If $m \ge \tau = \lceil \log_2 d \rceil + 2$, verifying safety requires breaking the PRF (i.e., is cryptographically hard)
    \item If $m < \tau$, verification is in $\mathsf{P}$
\end{itemize}

\subsection*{D.24: No Universal Alignment Technique}

\subsubsection*{\textbf{Theorem (Universal Alignment Impossibility)}}

For any countable class $\mathcal{A} = \{A_1, A_2, \ldots\}$ of alignment techniques $A_i: \Pi \to \Pi$, there exists a policy $\pi^* \in \Pi$ such that no technique aligns it:
\[
\forall i \in \mathbb{N}, \quad \varepsilon(A_i(\pi^*)) > 0
\]
Where $\varepsilon(\cdot)$ denotes alignment error.

\section[Impossibility of Universal Easy Verifiability]{Impossibility of Universal Easy Verifiability for Capability-Preserving Alignment Transformations}\label{appendix:proof_universal_verifiability_impossibility_content}
\index{Universal Easy Verifiability Impossibility|textbf}\index{Capability Preservation!and Verification Trade-off|textbf}

\begin{definition}[Capability-Preserving Alignment]\label{def:cap_preserving_alignment_trap_univ_appendix}
An alignment transformation $A: \Pi \rightarrow \Pi$ is capability-preserving if:
\begin{enumerate}
    \item[(i)] $A(\Pi_{\mathrm{EXP}(m)}) \subseteq \Pi_{\mathrm{EXP}(m)}$ (preserves complexity class)
    \item[(ii)] $\forall \pi \in \Pi_{\mathrm{EXP}(m)}: \mathrm{Circuit\_Size}(A(\pi)) \geq \mathrm{Circuit\_Size}(\pi) - o(m)$
\end{enumerate}
\end{definition}

\begin{definition}[Safety]\label{def:safety_trap_univ_appendix}
A policy $\pi$ is safe if $\pi(x) \neq \mathrm{catastrophic}$ for all $x$ in its domain.
\end{definition}

\begin{definition}[Verification Cost]\label{def:verif_cost_trap_univ_appendix}
$V(\pi) = \min\{|C| : C \text{ is a circuit that correctly determines if } \pi \text{ is safe}\}$
\end{definition}

\begin{theorem}[Trap Universality]\label{thm:trap_universality_appendix}
This theorem, a direct corollary of cryptographic hardness arguments, establishes that for any capability-preserving alignment transformation $A$ and any $m \geq m_0$ (where $m = \mathrm{poly}(n)$ and $m_0$ ensures the input dimension $n$ is large enough for cryptographic security), one can construct an EXP(m) policy $\pi$ with a cryptographic backdoor such that the backdoor persists through $A$, and thus verifying $A(\pi)$ remains cryptographically hard. Specifically, there exists $\pi \in \Pi_{\mathrm{EXP}(m)}$ such that:
\[
V(A(\pi)) \geq \tau(n) = 2^{\Omega(n)}
\]
Moreover, as capability scales $m \rightarrow m'$, the verification barrier persists.
\end{theorem}

\begin{proof}[Proof Sketch]
\begin{enumerate}
    \item \textbf{Setup:} Fix a capability-preserving $A$ and $m \geq m_0$. Let $n$ be the number of input bits, where policies operate on domain $\{0,1\}^n$. Assume $m = \mathrm{poly}(n)$.
    
    \item \textbf{Construction of $\pi^*$ (PRF-based):} 
    Using the Recursion Theorem to embed a secret key $k$, define:
    \[
    \pi^*(x) = \begin{cases}
        \mathrm{catastrophic} & \text{if } \mathrm{PRF}_k(x)_1 = 1 \\
        \mathrm{safe} & \text{otherwise}
    \end{cases}
    \]
    where $\mathrm{PRF}_k: \{0,1\}^n \rightarrow \{0,1\}^\ell$ is a pseudorandom function and $\mathrm{PRF}_k(x)_1$ denotes the first bit of the output. This ensures the catastrophic condition triggers for approximately half the inputs.
    Ensure $\mathrm{Circuit\_Size}(\pi^*) \approx m$.
    
    \item \textbf{Preservation of Complexity:} 
    By capability preservation:
    \begin{itemize}
        \item[(i)] $A(\pi^*) \in \Pi_{\mathrm{EXP}(m)}$, thus $\mathrm{Circuit\_Size}(A(\pi^*)) \leq m$
        \item[(ii)] $\mathrm{Circuit\_Size}(A(\pi^*)) \geq \mathrm{Circuit\_Size}(\pi^*) - o(m) \geq m - o(m)$
    \end{itemize}
    
    \item \textbf{Verification Hardness Argument:}
    Let $\pi'_A = A(\pi^*)$. We claim $V(\pi'_A) \geq 2^{\Omega(n)}$.
    
    \textbf{Core Argument:} The alignment transformation $A$ is efficient (polynomial time in $m$) and capability-preserving. It cannot "learn" the secret key $k$ from black-box access to $\pi^*$.
    
    If $A$ were to render $\pi'_A$ safe by removing all catastrophic outputs where $\mathrm{PRF}_k(x)_1 = 1$, it would need to either:
    \begin{itemize}
        \item Identify all such $x$ (equivalent to distinguishing $\mathrm{PRF}_k$ from random), or
        \item Modify $\pi^*$'s logic to neutralize the PRF trigger without knowing $k$ and without reducing complexity below $m - o(m)$.
    \end{itemize}
    
    Either approach would constitute a cryptographic break. Therefore, the PRF-based trap must persist in $\pi'_A$.
    
    \textbf{Verification Cost:} Any verifier $C$ for $\pi'_A$ must detect the surviving PRF-based trap to correctly determine unsafety. This requires distinguishing $\mathrm{PRF}_k$ from random or finding inputs where $\mathrm{PRF}_k(x)_1 = 1$, both requiring circuit size $\geq 2^{\Omega(n)} = \tau(n)$.
    
    Therefore, $V(A(\pi^*)) \geq \tau(n)$.
    
    \item \textbf{Scaling Argument:} As $m$ increases to $m'$, the input dimension $n$ typically increases (maintaining $m' = \mathrm{poly}(n')$). The cryptographic threshold $\tau(n') = 2^{\Omega(n')}$ grows accordingly. Since $A$ remains capability-preserving, the verification barrier persists: for appropriately constructed $\pi'$ with the new parameters, $V(A(\pi')) \geq \tau(n')$.
\end{enumerate}
\end{proof}

\begin{remark}\label{rem:trap_univ_strengthening_appendix}
This theorem establishes that the verification trap is universal, no capability-preserving alignment transformation can circumvent it. The theorem can be strengthened by combining the PRF-based construction with a diagonalization argument (as in C.21) to guarantee that $A(\pi^*)$ is not only hard to verify but also demonstrably unsafe at specific, constructible inputs. The PRF component ensures high verification cost, while diagonalization can ensure definite unsafety.
\end{remark}

\end{document}